\documentclass[manuscript,onefignum,onetabnum]{siamart190516}

\usepackage{graphicx}
\usepackage{subfig}
\usepackage{amssymb}
\usepackage{xcolor}
\usepackage{soul}
\usepackage{amsmath}
\usepackage{lineno}
\usepackage{graphics}
\usepackage{caption}
\usepackage{lscape}
\usepackage{tabularx}
\usepackage[makeroom]{cancel}
\usepackage[english]{babel}
\usepackage{float}
\usepackage{amsfonts}
\usepackage{textcomp}
\usepackage[utf8]{inputenc}
\usepackage[english]{babel}
\usepackage{mathtools}
\usepackage{lipsum}
\usepackage{algorithm}
\usepackage{todonotes}
\usepackage{algorithm}
\usepackage{algorithmic}
\usepackage{lipsum}
\usepackage{amsfonts}
\usepackage{graphicx}
\usepackage{epstopdf}
\usepackage{algorithmic}
\ifpdf
  \DeclareGraphicsExtensions{.eps,.pdf,.png,.jpg}
\else
  \DeclareGraphicsExtensions{.eps}
\fi

\newsiamremark{remark}{Remark}
\newsiamremark{hypothesis}{Hypothesis}
\crefname{hypothesis}{Hypothesis}{Hypotheses}
\newsiamthm{claim}{Claim}
\headers{Grassmannian diffusion maps}{K. R. M. dos Santos, D. G. Giovanis, and M. D. Shields }

\title{Grassmannian diffusion maps based dimension reduction and classification for high-dimensional data \thanks{Preprint submitted to SIAM Journal on Scientific Computing (SISC).}
}

\author{Ketson R. M. dos Santos\thanks{Department of Civil \& Systems Engineering, Johns Hopkins University, Baltimore, MD, 21218, USA.}
\and Dimitris G. Giovanis\footnotemark[2]
\and Michael D. Shields\footnotemark[2]
}

\usepackage{amsopn}

\ifpdf
\hypersetup{
  pdftitle={Grassmannian diffusion maps based classification for high-dimensional data},
  pdfauthor={K. R. M. dos Santos, D. G. Giovanis, and M. D. Shields}
}
\fi

\begin{document}

\maketitle

% REQUIRED
\begin{abstract}
% Diffusion Maps is a nonlinear dimensionality reduction technique used to embed high-dimensional data in a low-dimensional Euclidean space, where the notion of distance is due to the transition probability of a random walk over the dataset. However, the conventional approach is not capable to reveal the dataset underlying subspace structure, a useful information for machine learning applications such as object classification and face recognition. 
{This work introduces the Grassmannian Diffusion Maps, a novel nonlinear dimensionality reduction technique that defines the affinity between points through their representation as low-dimensional subspaces corresponding to points on the Grassmann manifold. The method is designed for applications, such as image recognition and data-based classification of constrained high-dimensional data where each data point itself is a high-dimensional object (i.e. a large matrix) that can be compactly represented in a lower dimensional subspace.  The GDMaps is composed of two stages. The first is a pointwise linear dimensionality reduction wherein each high-dimensional object is mapped onto the Grassmann manifold representing the low-dimensional subspace on which it resides. The second stage is a multi-point nonlinear kernel-based dimension reduction using Diffusion maps to identify the subspace structure of the points on the Grassmann manifold. To this aim, an appropriate Grassmannian kernel is used to construct the transition matrix of a random walk on a graph connecting points on the Grassmann manifold. Spectral analysis of the transition matrix yields low-dimensional Grassmannian diffusion coordinates embedding the data into a low-dimensional reproducing kernel Hilbert space. Further, a novel data classification/recognition technique is developed based on the construction of an overcomplete dictionary of reduced dimension whose atoms are given by the Grassmannian diffusion coordinates. Three examples are considered. First, a ``toy'' example shows that the GDMaps can identify an appropriate parametrization of structured points on the unit sphere. The second example demonstrates the ability of the GDMaps to reveal the intrinsic subspace structure of high-dimensional random field data. In the last example, a face recognition problem is solved considering face images subject to varying illumination conditions, changes in face expressions, and occurrence of occlusions. The technique presented high recognition rates (i.e., 95\% in the best-case) using a fraction of the data required by conventional methods.}

\end{abstract}

% REQUIRED

% REQUIRED

\section{Introduction}
\label{S:1}
Dimensionality reduction techniques play a fundamental role in the interpretation and characterization of high-dimensional data in computationally intensive data-driven applications such as data compression \cite{baraniuk2010,gan2020}, data classification \cite{wu2020,yan2020}, uncertainty quantification \cite{giovanis2018,soize2020,giovanis2020,wang2020}, and biological sciences \cite{uh2016, moon2017, moon2018}, just to mention few of them. Although linear dimensionality reduction techniques (e.g., principal component analysis (PCA) \cite{jolliffe2011}) suffice to extract features of the elements of a dataset, they fail to capture the intrinsic nonlinear geometric structure of the dataset itself. To circumvent this limitation, kernel-based techniques have been successfully employed in several applications by taking advantage of the data similarity expressed on a connected graph on the data. Among these techniques one can include isometric mapping, also known as Isomap \cite{tenenbaum2000,balasubramanian2002}; local linear embedding (LLE) \cite{roweis2000, donoho2003}; T-distributed stochastic neighbor embedding (t-SNE) \cite{maaten2008}; kernel principal component analysis (Kernel PCA) \cite{scholkopf1998}; Laplacian eigenmaps \cite{belkin2003}; diffusion maps (DMaps) \cite{coifman2006}; and vector diffusion maps (VDMaps) \cite{singer2012}.

Existing nonlinear dimensionality reduction techniques consider that high-dimen-\\sional data lie on a low-dimensional manifold embedded in a Euclidean space. Therefore, to reveal this embedded low-dimensional structure, one can resort to kernel-based techniques such as diffusion maps \cite{coifman2006}; where the spectral decomposition of the transition matrix of a random walk performed on the data is used to determine a new set of coordinates embedding this manifold into a low-dimensional Euclidean space. But, for certain classes of problems, the manifold on which the data lie is not well-characterized in Euclidean space. Instead, certain data sets (particularly very high-dimensional data sets such as images and numerical simualtions with many degrees of freedom) have structure that is inherently constrained to lie on a submanifold of a Reimmannian manifold.

Recent works have begun to consider that nonlinear dimensionality techniques can benefit from incorporating metrics beyond conventional Euclidean distance metrics. For example, an extension of DMaps, known as vector diffusion maps (VDMaps) \cite{singer2012}, includes the information about orthogonal transformations of neighboring points into the definition of similarity. Although VDMaps uses orthogonal transformations to align the data, similarity is still based on the Euclidean distance. The authors of \cite{singer2012} suggested that the Grassmannian distance could be considered; however, they note that using the Grassmannian distance in a Gaussian kernel could yield a non positive semi-definite kernel matrix \cite{harandi2014}. In this work, we demonstrate how a Grassmannian metric can be employed in the DMaps to explore the intrinsic structure of data constrained to lie on a submanifold of the Grassmann manifold (i.e. a Grassmannian diffusion manifold).

In summary, the Grassmann manifold \cite{griffiths2011,ye2016schubert} is a collection of all subspaces (of a particular dimension) of a vector space. Of particular importance is that the Grassmann manifold is endowed with a metric. Structured data, such as those from images or those from physics-based models, are often constrained to a particular submanifold of the Grassmannian. That is, there are only certain structured subspaces on which the data can lie given the constraints on the data. Several applications using Grassmann manifold projections are found in the literature; including the works by Turaga et al. \cite{turaga2008}, where statistical inference on the Grassmann manifold is investigated for enhancing the performance of activity recognition, video-based face recognition, and shape classification techniques. Moreover, Harandi et al. \cite{harandi2015} used the Grassmann manifold theory to address problems in sparse coding and dictionary learning. Giovanis and Shields \cite{giovanis2018,giovanis2020} introduced a computationally efficient surrogate modeling scheme based on Grassmann manifold projections for prediction of high-dimensional stochastic models in an uncertainty quantification setting.

In general, topics related to manifold learning and data similarity have been addressed by several authors. In this regard, Lin and Zha \cite{lin2008} focused on the formulation of the dimensionality reduction problem as a Riemannian geometry problem. In particular, a Riemannian normal coordinate chart is implemented, and a simplicial reconstruction of the manifold is performed. Therefore, a map from the high-dimensional into a low-dimensional space is computed such that the radial geodesic distances are preserved. Fan et al. \cite{fan2012} introduced a semi-supervised manifold learning technique known as Multi-Manifold Discriminative Analysis (Multi-MDA) to explore the information encoded by the geodesic distance. Feng et al. \cite{feng2013} has developed a global linear algorithm for dimensionality reduction, known as Maximal Similarity Embedding (MSE), which utilized the cosine metric to capture the intrinsic geometry of neighbor data aiming at the maximization of the global similarity. Wu et al. \cite{wu2014} introduced the Neighboring Similarity Integration (NSI) for the analysis of image manifold under the probability preserving principle. Zhao et al. \cite{zhao2018} propose a technique, referred to as multi-view manifold learning with locality alignment (MVML-LA), to learn a low-dimensional space containing a relevant information about the input data. Their main motivation was based on the fact that current manifold learning methods cannot be easily applied to multiple feature sets.

Herein, a novel dimensionality reduction technique is introduced for revealing the underlying geometric structure of a dataset on a Reimannian submanifold. The technique has two steps, the first is a linear pointwise dimensionality reduction via projection of the data onto the Grassmann manifold. Next, a multipoint nonlinear dimensionality reduction is conducted based on spectral analysis of the transition probability matrix of a random walk on the Grassmann manifold via diffusion maps with an appropriate Grassmannian kernel. This novel technique is particularly suitable for the classification and dimensionality reduction of high-dimensional data whose structure is constrained on a Grassmann manifold (e.g., images, videos, numerical simulations). Therefore, this technique is developed to be complementary to the existing ones such as DMaps and VDMaps, where no pointwise projections onto the Grassmann manifold are considered and measures of affinity in Euclidean spaces are sufficiently revealing. Moreover, a novel data classification/recognition is introduced using elements of sparse representation and the construction of an overcomplete dictionary composed by the Grassmannian diffusion coordinates.

% This paper is organized as follows. In section \ref{S:2}, elements Grassmann manifold theory are introduced. The definitions of Stiefel and Grassmann manifolds are presented together with the definition of the tangent space and geodesic path, as well as definitions of the exponential and logarithmic maps. In section \ref{S:3}, the notion of distance is introduced based on the concept of principal angles between subspaces. Section \ref{S:4} introduces the Grassmannian kernels as a necessary tool for the development of the multipoint dimensionality reduction. In section \ref{S:5}, the Grassmannian diffusion maps (GDMaps) is introduced, and the characteristics of the obtained diffusion coordinates are discussed. Section \ref{S:6} introduces a novel data classification/recognition technique informed by the Grassmannian diffusion maps, where accurate classifications are performed with a reduced amount of data in comparison with conventional techniques. In section \ref{S:7}, three examples are presented to corroborate the arguments presented in sections \ref{S:5} and \ref{S:6}.

\section{Grassmannian Diffusion Maps}

The proposed Grassmannian Diffusion Maps (GDMaps) is a two-stage dimension reduction specifically developed for problems where each data ``point'' is a high-dimensional object (vector, matrix, tensor). Identifying the geometric structure spanning a dataset composed of such points poses distinct challenges, specifically because the traditional notion of distance (i.e. a Euclidean metric) between these high-dimensional objects is difficult to interpret and may not even have useful meaning in illuminating the concepts of proximity or similarity \cite{aggarwal2001surprising}. Understanding the distance between these high-dimensional objects requires a deeper geometric interpretation. This geometry is formalized through analysis of the high-dimensional data points on the Grassmann manifold. Hence, the first stage is to project the data onto the Grassmann manifold, details of which are provided in Section 3.

Datasets composed of high-dimensional data with low-rank structures are known to possess an intrinsic connected geometric structure on the Grassmannian \cite{wang2015,wang2019}. That is, the data reside on a submanifold, the low-dimensional characterization of which is the objective of this work. In other words, the Grassmann manifold onto which we project the data contains a vast expanse of points that are geometrically unrelated to the dataset, and therefore are of no interest. We aim only to characterize the portion of the Grassmann manifold that meaningfully spans the dataset. For this, we apply the diffusion maps, which embeds the submanifold into a Hilbert Space. The diffusion maps, however, require us to specify a notion of distance on the Grassmann manifold, and from it derive a valid kernel. These topics are specifically discussed in Sections 4 - 5. Finally, after establishing a Grassmannian kernel, a spectral decomposition of the transition probability of a random walk constructed on a connected graph over the points on the Grassmannian enables the embedding of the submanifold and provides a convenient low-dimensional representation of the data set that can be used to interpret, classify, and otherwise analyze the dataset. This embedding is discussed in detail in Section 6.

Again, it is important to emphasize that the proposed GDMaps represents a natural extension of the DMaps, and the generalized VDMaps, for problems where the dataset is composed of high-dimensional objects whose geometry is not well-characterized in Euclidean spaces. It is not meant as a replacement for the DMaps or VDMaps for problems, such as the classical swiss roll, where proximity on the manifold can be well-characterized using Euclidean metrics. 

For the sake of clarity in its use and implementation, we begin by introducing the GDMaps algorithm. We then detail the mathematical foundations upon which the approach rests.

\subsection{Grassmann manifold projection}
\label{S:2}
The first stage of the proposed dimension reduction is to project the data onto the Grassmann manifold. Practically speaking, this involves identifying the numerical basis on which the data lie. Consider the data matrix $\mathbf{X}_i \in \mathbb{R}^{n \times m}$ having rank $p$ (generally $p\ll n,m$). A number of techniques exist for identifying a basis for $\mathbf{X}_i$ including the singular value decomposition (SVD) and QR decomposition. The resulting space spanned by this basis is a point on the Grassmann manifold and represents the projection of the data onto the manifold. Here, the SVD is preferred because it provides bases for both the column space and the row space of $\mathbf{X}_i$, which allows for feature recognition in both dimensions of the array. 

More specifically, consider the compact SVD
\begin{equation}
    \mathbf{X}_i = \mathbf{\Psi}_i\mathbf{\Sigma}_i\mathbf{\Phi}_i^T
\end{equation}
where $\mathbf{\Psi}_i\in \mathbb{R}^{n \times p}$ are the left singular vectors, $\mathbf{\Phi}_i\in \mathbb{R}^{m \times p}$ are the right singular vectors, and $\mathbf{\Sigma}_i\in \mathbb{R}^{p \times p}$ is a diagonal matrix contain the $p$ singular values. The matrices $\mathbf{\Psi}_i$ and $\mathbf{\Phi}_i$ define points on the Stiefel manifolds $\mathcal{V}(p,n)$ and $\mathcal{V}(p,m)$, and provide a (non-unique) representation for the points $\mathcal{X}_i$ and $\mathcal{Y}_i$ on the Grassmann manifolds $\mathcal{G}(p,n)$ and $\mathcal{G}(p,m)$ defined as the spaces spanned by $\mathbf{\Psi}_i$ and $\mathbf{\Phi}_i$. These matrices serve as the projection of the $\mathbf{X}_i$ onto the Grassmann manifold, and represent the dimension reduced data on which we further operate. Definitions and properties of the Stiefel and Grassmann manifolds, including the essential operations on data on the manifold are presented in Section \ref{sec:Grassmann}. 

% \subsection{Diffusion Maps on the Grassmann manifold}

\subsection{Diffusion Maps: Discrete embedding of Grassmannian submanifold on Euclidean space}
\label{S:5.2}
Given a set $S_N = \{\mathbf{X}_i\}^N_{i=1}$ of i.i.d.\ random high-dimensional samples $\mathbf{X}_i \in \mathbb{R}^{n \times m}$ with probability distribution $f_x$, and their projections onto the Grassmann manifold $\mathcal{G}(p,n)$ given by the set $G_N = \{\mathcal{X}_i\}^N_{i=1}$, one can construct a graph on $\mathcal{G}(p,n)$ where the nodes are given by $G_N$ and the weights on the edges are determined by a positive semi-definite Grassmannian kernel $k: \mathcal{G}(p,n) \times \mathcal{G}(p,n) \rightarrow \mathbb{R}$. Essentials for defining and selecting a Grassmannian kernel are discussed in Sections \ref{S:4} and \ref{S:5} where we review the definitions of distances, metrics, and kernels on the Grassmannian. 

The transition matrix of a random walk $W_N = \left(S_N,f,\mathbf{P}\right)$ performed in this graph can be constructed from the kernel normalization. First, a diagonal degree matrix $\mathbf{D} \in \mathbb{R}^{N \times N}$ is constructed as follows
\begin{equation}\label{eq:5.21}
    D_{ii} = \sum_{j=1}^{N} k(\mathcal{X}_i,\mathcal{X}_j).
\end{equation}
\noindent
from which one can obtain the stationary distribution of the random walk as
\begin{equation}\label{eq:5.22}
    \pi_{i} = \frac{D_{ii}}{\sum_{k=1}^N D_{kk}}.
\end{equation}

Next, a set of coordinates embedding the points on $\mathcal{G}(p,n)$ into an Euclidean space is obtained from the spectral decomposition of the transition matrix $\mathbf{P}$ of random walk $W_N$. To obtain this transition matrix, the kernel matrix $k_{ij} = k(\mathcal{X}_i,\mathcal{X}_j)$ must be normalized as follows
\begin{equation}\label{eq:5.23}
    \kappa_{ij} = \frac{k_{ij}}{\sqrt{D_{ii}D_{jj}}},
\end{equation}
\noindent
and $P_{ij}$ is given by
\begin{equation}\label{eq:5.24}
    P_{ij} = \frac{\kappa_{ij}}{\sum_{k=1}^{N} \kappa_{ik}}.
\end{equation}
\noindent
From the eigendecomposition of $\mathbf{P}$, we retain the first $q$ eigenvectors $\left\{\psi_k\right\}^q_{k=0}$, with $\psi_k \in \mathbb{R}^N$, and their respective eigenvalues $\left\{\lambda_k\right\}^q_{k=0}$. The Grassmannian diffusion coordinates for the element $\mathbf{X}_j \in S_N$ are then given by 
\begin{equation}\label{eq:5.25}
    \mathbf{\Xi}_j = \left(\xi_{j0}, \dots, \xi_{jq} \right) = \left(\lambda_0 \psi_{j0}, \dots, \lambda_1 \psi_{jq} \right),
\end{equation}
\noindent
where $\psi_{jk}$ corresponds to the position $j$ of $\psi_k$. Moreover, the Grassmannian diffusion distance for the discrete embedding of the Grassmann manifold in the Euclidean space is given by
\begin{equation}\label{eq:5.25}
    \delta^{\mathcal{G}}_{ij} = ||P^t_{ik} - P^t_{jk}||_{L_2(D^{-1}_{ii})} = \Bigg\{\sum_{k=1}^N\left[P^t_{ik} - P^t_{jk}\right]^2 \frac{1}{D_{kk}} \Bigg\}^{\frac{1}{2}},
\end{equation}
\noindent
which can be interpreted as
\begin{equation}\label{eq:5.26}
    \delta^{\mathcal{G}}_{ij} = \Bigg\{\sum_{k=1}^N\lambda_k^{2t}\left[\psi_{ik} - \psi_{jk}\right]^2 \Bigg\}^{\frac{1}{2}}.
\end{equation}
One can easily see that the Grassmannian diffusion distance can be determined as the Euclidean distance between the Grassmannian diffusion coordinates as given by
\begin{equation}\label{eq:gr_dist}
    \delta^{\mathcal{G}}_{ij} = \|\mathbf{\Xi}_i - \mathbf{\Xi}_j\|^2.
\end{equation}

The mechanization of the Grassmannian diffusion maps is presented later in Algorithm \ref{alg:5.1}.

\begin{algorithm}[h]
\caption{Grassmannian Diffusion Maps (GDMaps)}
\label{alg:5.1}
\begin{algorithmic}[1]
\REQUIRE a set of $N$ high-dimensional data $S_N = \left\{ \mathbf{X}_1, \dots \mathbf{X}_N\right\}$ with $\mathbf{X}_i \in \mathbb{R}^{n \times m}$, and the dimension $p$ of the Grassmann manifold.
\FOR{$i \in 1, \dots, N$}
\STATE Compute the compact Singular Value Decomposition: $\mathbf{X}_i = \mathbf{\Psi}_i\mathbf{\Sigma}_i\mathbf{\Phi}_i^T$, where $\mathbf{\Psi}_i \in \mathcal{V}(p,n)$ and $\mathbf{\Phi}_i \in \mathcal{V}(p,m)$.
\ENDFOR
\STATE For every pair $\left[\mathbf{\Psi}_i,\mathbf{\Psi}_j\right]$ and $\left[\mathbf{\Phi}_i,\mathbf{\Phi}_j\right]$ compute the entries $k_{ij}$ of the kernel matrices $k_{ij}\left(\mathbf{\Psi}\right)$ and $k_{ij}\left(\mathbf{\Phi}\right)$, either e.g.\ Eq. (\ref{eq:4.2}) or Eq. (\ref{eq:4.4}).
\STATE If necessary, compute a composite kernel matrix $k\left(\mathbf{\Psi},\mathbf{\Phi}\right)$. For example:\\
\vspace{0.2cm}
$k\left(\mathbf{\Psi},\mathbf{\Phi}\right) = k_{ij}\left(\mathbf{\Psi}\right)+k_{ij}\left(\mathbf{\Phi}\right)$ or $k\left(\mathbf{\Psi},\mathbf{\Phi}\right) = k_{ij}\left(\mathbf{\Psi}\right) \circ k_{ij}\left(\mathbf{\Phi}\right)$, where $\circ$ is the Hadamard product.
\vspace{0.2cm}
\STATE Compute degree matrix $\mathbf{D} \in \mathbb{R}^{N \times N}$ using $k\left(\mathbf{\Psi},\mathbf{\Phi}\right)$ as in Eq. (\ref{eq:5.21}).
\STATE Compute the normalized kernel matrices using Eq. (\ref{eq:5.23}) for $k\left(\mathbf{\Psi},\mathbf{\Phi}\right)$.
\STATE Estimate the transition matrix $\mathbf{P}$ using Eq. (\ref{eq:5.24}).
\STATE Obtain the eigenvectors and their respective eigenvalues from the eigendecomposition of $\mathbf{P}$, and determine the truncation index $q$.
\ENSURE diffusion coordinates $\mathbf{\Xi}_1, \dots, \mathbf{\Xi}_N$, with $\mathbf{\Xi}_i \in \mathbb{R}^{q}$.
\end{algorithmic}
\end{algorithm}

\section{Grassmann manifold: Definitions}
\label{sec:Grassmann}
The elements of differential geometry discussed in this section are essential for the development of the Grassmannian diffusion maps technique. In particular, the high-dimensional data considered herein are represented by matrices (data in vector or tensor form can be reshaped to matrix from), and their projection on the Grassmann manifold are represented by their underlying subspace structure. In this regard, to formalize the representation of a point on the Grassmann manifold it is necessary to define both the Stiefel and the Grassmann manifolds.

Consider the ambient space $\mathbb{R}^n$, a $p$-plane is the subspace of dimension $p$ with $0 < p < n$, and a $p$-frame is an ordered set of $p$ mutually orthonormal vectors in $\mathbb{R}^n$ \cite{auslander2012,ye2016schubert}. The Stiefel manifold $\mathcal{V}(p,n)$, which is induced by the orthogonal group $O(n)$ of $n\times n$ orthogonal matrices, is defined as follows.
\begin{definition}\label{def:2.1}
The Stiefel manifold $\mathcal{V}(p,n)$ is the set of $p$-frames in $\mathbb{R}^n$ such that $\mathcal{V}(p,n) = \{\mathbf{\Psi} \in \mathbb{R}^{n \times p}: \mathbf{\Psi}^\intercal\mathbf{\Psi} = \mathbf{I}_p\}$.
\end{definition}
\noindent
where $\mathbf{I}_p \in \mathbb{R}^{p \times p}$ is the identity matrix and $\mathbf{\Psi} \in \mathbb{R}^{n \times p}$ is an orthonormal matrix. Moreover, $\mathcal{V}(p,n)$ is a compact manifold with dimension given by $\mathrm{dim}\left[\mathcal{V}(p,n)\right] = np - \frac{1}{2}p(p+1)$ \cite{ye2016schubert}. Furthermore, the Stiefel manifold is a homogeneous space represented as a quotient space \cite{ye2016schubert,ye2019}, such that
\begin{equation}\label{eq:2.1}
    \mathcal{V}(p,n) \cong \frac{O(n)}{O(n-p)}.
\end{equation}

The right action of $O(p)$ on $\mathcal{V}(p,n)$ induces a homogeneous space with dimension $p(n-p)$ known as the Grassmann manifold \cite{ye2016schubert}, whose definition is given next.

\begin{definition}\label{def:2.2}
The Grassmann manifold (or Grassmannian) $\mathcal{G}(p,n)$ is a set of $p$-planes in $\mathbb{R}^n$ where a point on $\mathcal{G}(p,n)$ is given by $\mathcal{X} = \mathrm{span}\left(\mathbf{\Psi}\right)$ with $\mathbf{\Psi} \in \mathcal{V}(p,n)$.
\end{definition}

\noindent
Moreover, $\mathcal{X}$ is identified as an equivalence class of $n \times p$ matrices under orthogonal transformation of the Stiefel manifold \cite{ye2016schubert,ye2019}, such that.
\begin{equation}\label{eq:2.2}
    \mathcal{G}(p,n) \cong \frac{O(n)}{O(n-p)O(p)} = \frac{\mathcal{V}(p,n)}{O(p)}.
\end{equation}

For a better comprehension about the nature of both manifolds, if $p=1$ the Grassmann manifold $\mathcal{G}(1,n)$ is a generalization of the projective space $\mathbb{P}^{n-1}$ corresponding to the set of lines passing through the origin of the Euclidean space \cite{harandi2014}. Further, a point $\mathcal{X} = \mathrm{span}\left(\mathbf{\Psi}\right) \in \mathcal{G}(p,n)$ is invariant under the right action $\mathbf{\Psi R}\in \mathcal{V}(p,n)$  such that $\mathrm{span}\left(\mathbf{\Psi}\right) = \mathrm{span}\left(\mathbf{\Psi R}\right)$, where $\mathbf{R} \in O(p)$ is a $p\times p$ orthogonal matrix \cite{ye2016schubert}. Hence, $\mathbf{\Psi}$ and $\mathbf{\Psi R}$ designate the same point on the Grassmann manifold which is an equivalence class represented by either point on the Stiefel manifold. Practically speaking, this means that any matrix $\mathbf{X} \in \mathbb{R}^{n \times m}$ can be projected onto a Grassmann manifold $\mathcal{G}(p,n)$ through an orthonormalization process where only $p$ directions are kept and that the point on $\mathcal{G}(p,n)$ is unique despite the non-uniqueness of the orthonormalization process.

% Therefore, a representation of a point on the Grassmann manifold is given by an orthonormal matrix $\mathbf{\Psi} \in \mathbb{R}^{n \times p}$. 

% In fact, this orthonormal matrix is a point on $\mathcal{V}(p,n)$ (basis) representing a point on $\mathcal{G}(p,n)$ (subspace).

\subsection{Grassmann manifold: Tangent space and geodesic path}
\label{S:2.1}
Due to the smoothness of $\mathcal{G}(p,n)$, one can define tangent vectors at a given point $\mathcal{X} \in \mathcal{G}(p,n)$ as an equivalence class of differentiable curves $\gamma(t)$ passing through $\mathcal{X}$. In this regard, a tangent space $\mathcal{T}_{\mathcal{X}}\mathcal{G}(p,n)$ is defined as follows \cite{maruskin2012,wang2015,sommer2020}. 
\begin{definition}
The tangent space $\mathcal{T}_{\mathcal{X}}\mathcal{G}(p,n)$ is the set of all tangent vectors in $\mathcal{X}$, such that
\begin{equation}\label{eq:2.3}
    \mathcal{T}_{\mathcal{X}}\mathcal{G}(p,n) = \{\mathbf{\Gamma} \in \mathbb{R}^{n \times p} : \mathbf{\Gamma}^T\mathbf{\Psi}=\mathbf{0}\}.
\end{equation}
\end{definition}
The tangent space $\mathcal{T}_{\mathcal{X}}\mathcal{G}(p,n)$ at any point $\mathcal{X}$ is a flat inner-product space, which means that vectors can be defined connecting points in the tangent space, while such vectors cannot be defined directly on the manifold itself. Moreover, any such vector connecting points in the tangent space can be projected onto the manifold onto a curve known as the geodesic, $\gamma(t)$ on $\mathcal{G}(p,n)$, defined as follows.
% Moreover, the geodesic curve $\gamma(t)$ on $\mathcal{G}(p,n)$ has the following definition
\begin{definition}\label{def:2.3}
The geodesic curve $\gamma(t): I \rightarrow \mathcal{G}(p,n)$, is a differentiable curve on $\mathcal{G}(p,n)$ that is locally length-minimizing  with respect to a Riemannian metric. 
\end{definition}
\noindent
where $I$ is an interval in $\mathbb{R}$. Further, the vectors tangent to $\gamma(t)$ are covariantly constant, where $\nabla_{\dot{\gamma}}\dot{\gamma}(t)=0$ \cite{zimmermann2019}. In simpler terms, the geodesic is the shortest path along the manifold between two points $\mathcal{X}_0$ and $\mathcal{X}_1$.

To develop a map from the manifold to the tangent space and vice-versa, let's first restrict the interval $I \in \mathbb{R}$ to $I = [0, t]$, yielding a geodesic segment joining $\gamma(0)$ to $\gamma(t)$. Therefore, using the Einstein summation convention and the Levi-Civita connection, one can make the covariant derivative of $\dot{\gamma}$ equal to zero. Thus, one can obtain the following geodesic equation \cite{hartman1950,sommer2020}
\begin{equation}\label{eq:2.5}
    \ddot{\gamma}^\lambda + \Gamma^{\lambda}_{\mu \nu}\dot{\gamma}^\mu\dot{\gamma}^\nu=0,
\end{equation}
\noindent
where $\Gamma^{\lambda}_{\mu \nu}$ are the Christoffel symbols. Due to the local existence and uniqueness theorem for geodesics, one can say that given a point $\mathcal{X}_0 \in \mathcal{G}(p,n)$ and a vector $\dot{\mathcal{X}}_0 \in \mathcal{T}_{\mathcal{X}_0}\mathcal{G}(p,n)$, the geodesic exists and it is unique, such that $\gamma(0) = \mathcal{X}_0$ and $\dot{\gamma}(0) = \dot{\mathcal{X}}_0$. This can be proven using the theory of ordinary differential equations \cite{hartman1950,oprea2007}.

\subsection{Grassmann manifold: Exponential and logarithmic maps}
\label{S:2.2}
Considering that the Grassmann manifold $\mathcal{G}(p,n)$ is connected and complete as a metric space, one can define an exponential map $\mathrm{exp}_\mathcal{X}:\mathcal{T}_{\mathcal{X}}\mathcal{G}(p,n) \rightarrow \mathcal{G}(p,n)$ from the tangent space to every point $\mathcal{X} \in \mathcal{G}(p,n)$. Consider two points $\mathcal{X}_0 = \mathrm{span}(\mathbf{\Psi}_0)$ and $\mathcal{X}_1 = \mathrm{span}(\mathbf{\Psi}_1)$ in $\mathcal{G}(p,n)$, and a tangent space $\mathcal{T}_{\mathcal{X}_0}\mathcal{G}(p,n)$ with $\mathbf{\Gamma} \in \mathcal{T}_{\mathcal{X}_0}\mathcal{G}(p,n)$. One can map $\mathbf{\Gamma}$ to $\gamma(1) = \mathcal{X}_1$, where $\gamma(0) = \mathcal{X}_0$ and $\dot{\gamma}(0) = \dot{\mathcal{X}}_0$ (see Fig. \ref{fig:2}) \cite{giovanis2018,sommer2020}. As $\mathcal{X}_1$ is represented by the orthonormal matrix $\mathbf{\Psi}_1$, one can write the exponential map as
\begin{equation}\label{eq:2.6}
    \mathrm{exp}_{\mathcal{X}_0}(\mathbf{\Gamma}) = \mathbf{\Psi}_1.
\end{equation}
\begin{figure}[tbhp]
	\centering
	\captionsetup{justification=centering}
	\includegraphics[scale=0.08]{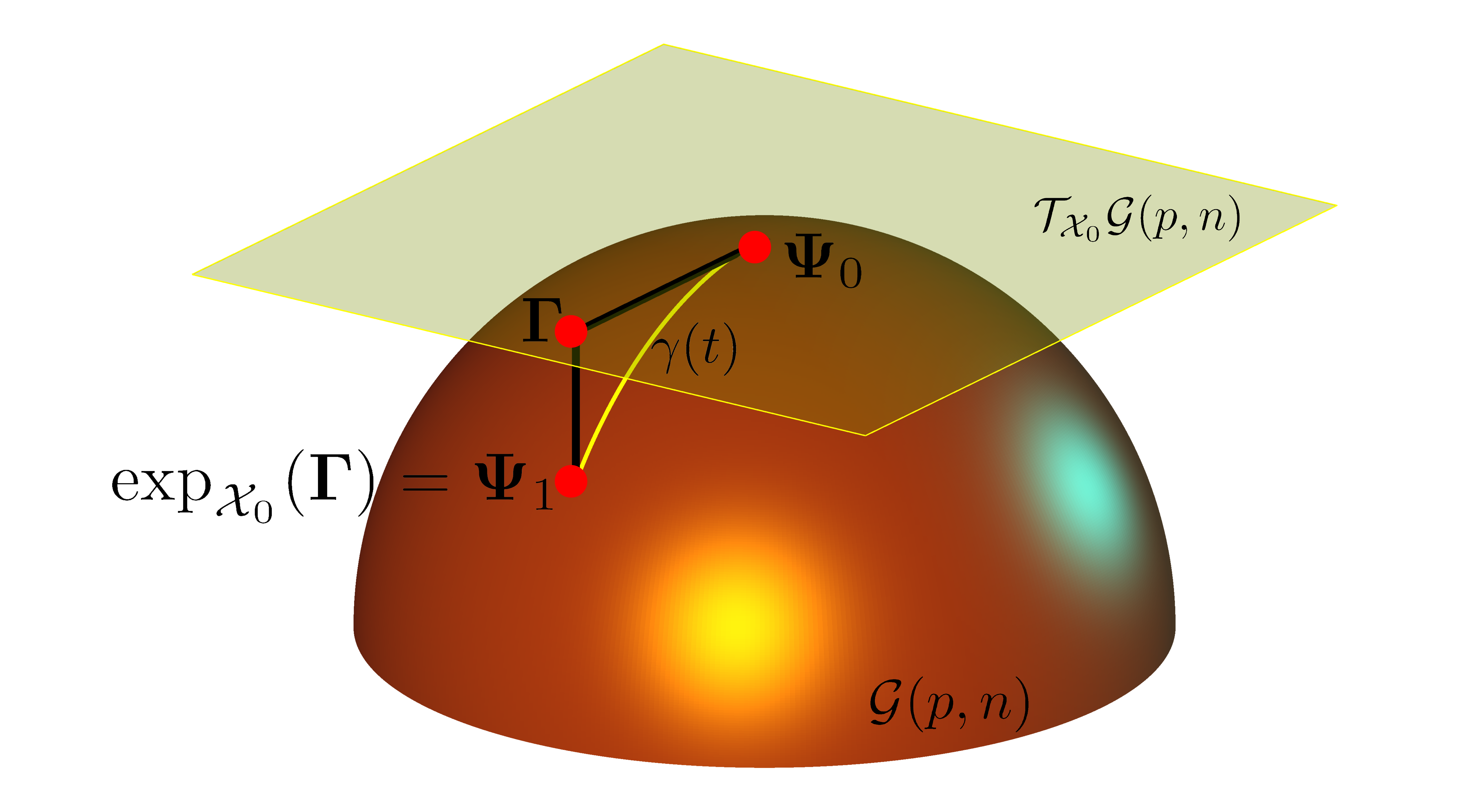}
	\vspace{-0.5em}
	\caption{Exponential map of a point $\mathbf{\Gamma} \in \mathcal{T}_{\mathcal{X}}\mathcal{G}(1,3)$ to the Grassmann manifold $\mathcal{G}(1,3)$.}
	\vspace{-0.5em}
	\label{fig:2}
\end{figure}

\noindent
Expressing $\mathbf{\Gamma}$ by its thin singular value decomposition $\mathbf{\Gamma} = \mathbf{U}\mathbf{S}\mathbf{V}^T$, one can obtain the following expression after some mathematical manipulation \cite{begelfor2006}.
\begin{equation}\label{eq:2.10}
    \mathbf{U}\mathrm{tan}\left(\mathbf{S}\right)\mathbf{V}^T = \left(\mathbf{\Psi}_1 -  \mathbf{\Psi}_0\mathbf{\Psi}_0^T\mathbf{\Psi}_1\right)\left(\mathbf{\Psi}_0^T\mathbf{\Psi}_1\right)^{-1}.
\end{equation}
\noindent
Defining the matrix $\mathbf{M}=\left(\mathbf{\Psi}_1-\mathbf{\Psi}_0\mathbf{\Psi}_0^T\mathbf{\Psi}_1\right)\left(\mathbf{\Psi}_0^T\mathbf{\Psi}_1\right)^{-1}$ and its thin SVD by $\mathbf{M}=\mathbf{U}\mathbf{S}\mathbf{V}^T$, the logarithmic map $\mathrm{log}_\mathcal{X}:\mathcal{G}(p,n) \rightarrow \mathcal{T}_{\mathcal{X}}\mathcal{G}(p,n)$, which is only invertible in the area close to $\mathcal{X}_0$ \cite{sommer2020}, is given by
\begin{equation}\label{eq:2.11}
    \mathrm{log}_\mathcal{X}(\mathbf{\Psi}_1) = \mathbf{U}\mathrm{tan}^{-1}\left(\mathbf{S}\right)\mathbf{V}^T.
\end{equation}

The geodesic $\gamma(t)$ parameterizes the curve connecting $\mathbf{\Psi}_0$ and $\mathbf{\Psi}_1$ on $t\in[0,1]$ with $\gamma(0) = \mathbf{\Psi}_0$ and $\gamma(1)=\mathbf{\Psi}_1$ such that the their respective projections in the tangent space are connected by a straight line. It can thus be represented by the projection of this line in the tangent space through the exponential mapping, where $\mathbf{\Gamma}$ is expressed by its thin singular value decomposition $\mathbf{\Gamma} = \mathbf{U}\mathbf{S}\mathbf{V}^T$ \cite{begelfor2006}. Therefore, one can write $\gamma(t)$ as
\begin{equation}\label{eq:2.12}
    \gamma(t)=\mathrm{span}\left[\left(\mathbf{\Psi}_0\mathbf{V}\mathrm{cos}(t\mathbf{S})+\mathbf{U}\mathrm{sin}(t\mathbf{S})\right)\mathbf{V}^T\right].
\end{equation}

\subsection{Grassmann manifold: Distances and metrics}
\label{S:4}
Points on the Grassmann manifold are intrinsically connected by smooth curves along which one can define a proper notion of dissimilarity given by the distance. This quantity is estimated from the principal angles between points on the Grassmann manifold, defined as follows.

\begin{definition}\label{def:3.1}
Considering $\mathbf{u}_i \in \mathrm{span}\left(\mathbf{\Psi}_u\right)$ and $\mathbf{v}_i \in \mathrm{span}\left(\mathbf{\Psi}_v\right)$ on $\mathcal{G}(k,n)$ and $\mathcal{G}(l,n)$, respectively, and letting $p = \mathrm{min}(k,l)$; the principal angles $0 \le \theta_1 \le \theta_2 \le \dots \le \theta_p \le \pi/2$ are recursively obtained from $\mathrm{cos}(\theta_i) = \underset{\mathbf{u}_i}{\mathrm{max}}\ \underset{\mathbf{v}_i}{\mathrm{max}}\ \mathbf{u}_i^T \mathbf{v}_i$ where $\mathbf{u}_i$ and $\mathbf{v}_i$ are orthonormal vectors.
\end{definition}

\noindent
The cosine of the principal angles $\theta_i \in \left[0, \pi/2\right]$ between the subspaces $\mathrm{span}(\mathbf{\Psi}_u)$ and $\mathrm{span}(\mathbf{\Psi}_v)$ can be computed from the singular values of $\mathbf{\Psi}_u^T\mathbf{\Psi}_v$ as:
\begin{equation}\label{eq:3.1}
    \mathbf{\Psi}_u^T\mathbf{\Psi}_v = \mathbf{U}\mathbf{S}\mathbf{V}^T,
\end{equation}
\noindent
where $\mathbf{U} \in O(k)$, $\mathbf{V} \in O(l)$, $\mathbf{S} = \mathrm{diag}(\sigma_1, \sigma_2, \dots, \sigma_p)$, with $p = \mathrm{min}(k,l)$, and the principal angles $\theta_i = \mathrm{cos}^{-1}(\sigma_i)$ \cite{miao1992}. In fact, it has been shown that any measure of distance on the Grassmann manifold must be a function of the principal angles as stated in the following theorem (repeated from \cite{wong1967,ye2016schubert}).
\begin{theorem}
Any notion of distance between $k$-dimensional subspaces in $\mathbb{R}^n$ that depends only on the relative positions of the subspaces, i.e., invariant under any rotation in $O(n)$, must be a function of their principal angles. To be more specific, if a distance $d:\mathcal{G}(k,n) \times \mathcal{G}(k,n) \to [0, \infty)$ satisfies $d(Q \cdot \boldsymbol{\Psi}_0, Q \cdot \boldsymbol{\Psi}_1) = d(\boldsymbol{\Psi}_0, \boldsymbol{\Psi}_1)$ for all $d(\boldsymbol{\Psi}_0, \boldsymbol{\Psi}_1) \in \mathcal{G}(k,n)$ and all $Q \in O(n)$, where $Q \cdot \boldsymbol{\Psi}_0 := \mathrm{span}(Q\boldsymbol{\Psi}_0) \in \mathcal{G}(k,n)$, then $d$ must be a function of $\theta_i(\boldsymbol{\Psi}_0, \boldsymbol{\Psi}_1), i = 1,...,k$.
\end{theorem}

% \subsection{Distance definitions}
% \label{S:3.1}
Perhaps the most common distance, the geodesic (or arc-length) distance $d_{\mathcal{G}(p,n)}\left(\mathbf{\Psi}_0,\mathbf{\Psi}_1\right)$ between two points $\mathcal{X}_0 = \mathrm{span}\left(\mathbf{\Psi}_0\right)$ and $\mathcal{X}_1 = \mathrm{span}\left(\mathbf{\Psi}_1\right)$ on $\mathcal{G}(p,n)$, corresponds to the distance over the geodesic $\gamma(t)$ parameterized by $t \in [0, 1]$, and it is given by \cite{wong1967,ye2016schubert,giovanis2018} 
\begin{equation}\label{eq:3.2}
    d_{\mathcal{G}(p,n)}\left(\mathbf{\Psi}_0,\mathbf{\Psi}_1\right) = \|\mathbf{\Theta}\|_2,
\end{equation}
\noindent
where $\mathbf{\Theta} = \left(\theta_1, \theta_2, \dots, \theta_p \right)$ is the vector of principal angles. This notion of distance, using the arc-length metric, between two subspaces in $\mathcal{G}(2,3)$ is represented on the unit semi-sphere in Fig. \ref{fig:3}. Several definitions of distance/metrics on $\mathcal{G}(p,n)$ can be found in the literature (see \cite{ye2016schubert} for  detailed information) and are listed in Table \ref{tab:1}.
\begin{figure}[tbhp]
	\centering
	\captionsetup{justification=centering}
	\includegraphics[scale=0.08]{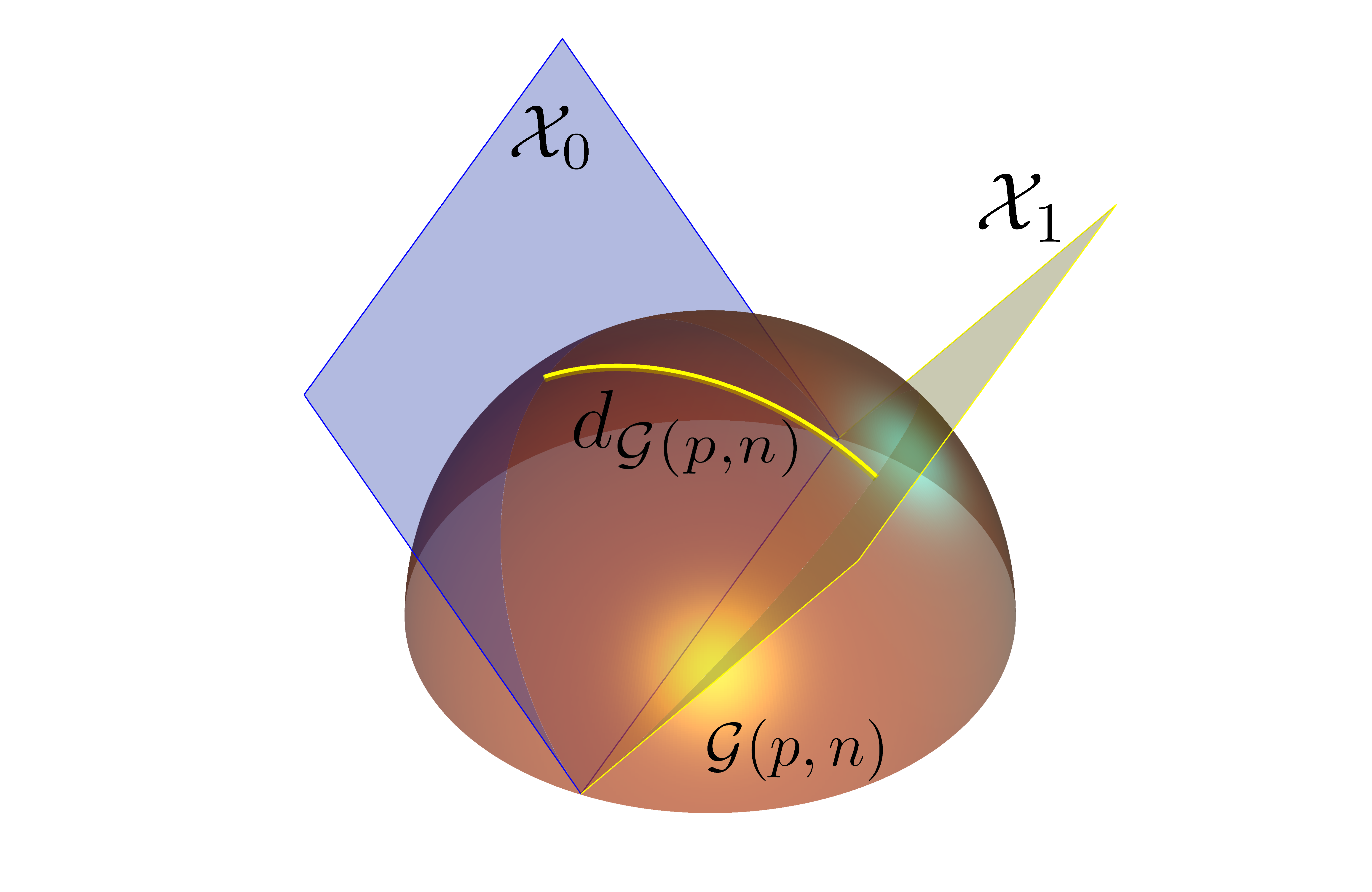}
	\vspace{-0.5em}
	\caption{Geodesic distance between subspaces $\mathcal{X}_0$ and $\mathcal{X}_1$ of $\mathbb{R}^2$ on $\mathcal{G}(2,3)$.}
	\vspace{-0.5em}
	\label{fig:3}
\end{figure}

\begin{table}[H]
\caption{Distances/metrics on the Grassmann manifold}
\label{tab:1}
\centering
\begin{tabular}{lcc}
\multicolumn{1}{c}{} & \textit{Principal angles} & \textit{Orthonormal basis}  \\
Asimov               &  $\theta_p$      &  $\mathrm{cos}^{-1}\|\mathbf{\Psi}_0^T\mathbf{\Psi}_1\|_2$  \\
Binet-Cauchy         & $\left(1-\prod^p_{i=1}\mathrm{cos}^2\theta_i\right)^{1/2}$           & $\left[1-\left(\mathrm{det}\mathbf{\Psi}_0^T\mathbf{\Psi}_1\right)^2 \right]^{1/2}$  \\
Arc-length              &  $\left(\sum^p_{i=1}\theta^2_i\right)^{1/2}$  & $\|\mathrm{cos}^{-1}\mathbf{\Sigma}\|_F$                  \\
Chordal              &  $\left(\sum^p_{i=1}\mathrm{sin}^2\theta_i\right)^{1/2}$  & $\frac{1}{\sqrt{2}}\|\mathbf{\Psi}_0\mathbf{\Psi}^T_0 - \mathbf{\Psi}_1\mathbf{\Psi}^T_1\|_F$\\
Procrustes           & $\left(2\sum_{i=1}^{p}\mathrm{sin}^2\frac{\theta_i}{2}\right)^{1/2}$ & $\|\mathbf{\Psi}_0 \mathbf{U} - \mathbf{\Psi}_1 \mathbf{V} \|_F$ \\
Projection           & $\mathrm{sin}\theta_p$   & $\|\mathbf{\Psi}_0\mathbf{\Psi}^T_0 - \mathbf{\Psi}_1\mathbf{\Psi}^T_1 \|_2$ \\
Spectral             & $2\mathrm{sin}\frac{\theta_p}{2}$ & $\|\mathbf{\Psi}_0 \mathbf{U} - \mathbf{\Psi}_1 \mathbf{V} \|_2$
\end{tabular}
\end{table}

When defining a distance, it is often important that the distance be provided in terms of a metric, which formalizes the notion of distance on the Grassmann manifold, defined as follows. 

\begin{definition}
A metric on the Grassmann manifold is a function $d:\mathcal{G}(k,n) \times \mathcal{G}(k,n) \to [0, \infty)$ where 
$[0,\infty )$ is the set of non-negative real numbers and for all $\mathcal{X}_0 = \mathrm{span}\left(\mathbf{\Psi}_0\right), \mathcal{X}_1 = \mathrm{span}\left(\mathbf{\Psi}_1\right),\mathcal{X}_2 = \mathrm{span}\left(\mathbf{\Psi}_2\right)  \in \mathcal{G}(k,n)$, the following three conditions are satisfied: 
\begin{enumerate}
    \item $d(\mathcal{X}_0,\mathcal{X}_1)=0 \Leftrightarrow \mathcal{X}_0 = \mathcal{X}_1 $ (identity of indiscernibles)
    \item $d(\mathcal{X}_0,\mathcal{X}_1) = d(\mathcal{X}_1,\mathcal{X}_0)$ (symmetry) 
    \item $d(\mathcal{X}_0,\mathcal{X}_1) \le d(\mathcal{X}_0,\mathcal{X}_2) + d(\mathcal{X}_2,\mathcal{X}_1)$ (triangle inequality)
\end{enumerate}
\end{definition}

The formalism of a metric on the manifold is useful for the definition of distance. It can be seen, for example, by observing that $d(\mathcal{X}_0,\mathcal{X}_1) = \sin \theta_1$, also known as max correlation distance or spectral distance \cite{ye2019}, is not a metric on the Grassmann manifold because it can be zero for a pair of distinct subspaces. Clearly this is undesirable.

\section{Grassmannian kernels}
\label{S:5}
% \textcolor{red}{Grassmannian kernels are important for the development of the Grassmannian diffusion maps (see Section \ref{S:5}) because they embed the Grassmann manifold into a reproducing kernel Hilbert space (RKHS). Moreover, they encode the notion of similarity, which is maximized when the dissimilarity (distance) is zero. In this section, a positive semi-definite kernel is defined as}

The notion of a metric on the Grassmann manifold is useful as an assessment tool to evaluate the proximity of points (subspaces) on the manifold. However, it is not sufficient as a tool for achieving an additional dimension reduction on a collection of points on a manifold as is necessary, for example, for clustering and classification of points on the manifold. This additional dimension reduction requires embedding the points on the manifold into a higher-dimensional “feature” space (the so-called “kernel trick” common to many nonlinear dimension reduction methods). For this, a kernel is required. Defining a Grassmannian kernel enables the Grassmann manifold to be embedded in a reproducing kernel Hilbert space (RKHS) that encodes the notion of similarity. Here, we begin by defining Grassmannian kernels, which follows from the following general definition of a kernel.

\begin{definition}\label{def:4.1}
A real symmetric map is a real-valued positive semi-definite kernel if $\sum_{i,j}c_i c_j k(x_i, x_j) \leq 0$, with $x \in \mathcal{X}$ and $c_i \in \mathbb{R}$.
\end{definition}
Consequently, one can define a Grassmannian kernel as

\begin{definition}\label{def:4.2}
A map $k: \mathcal{G}(p,n) \times \mathcal{G}(p,n) \rightarrow \mathbb{R}$ is a Grassmannian kernel if it is invariant to the choice of basis and positive semi-definite.
\end{definition}

In kernel-based dimensionality reduction techniques, there are many commonly used kernels defined on Euclidean spaces. The Gaussian kernel is perhaps the most popular and is given by,
\begin{equation}\label{eq:4.1}
    k(\mathbf{X}_i,\mathbf{X}_j) = \mathrm{exp}\left( -\frac{||\mathbf{X}_i - \mathbf{X}_j||^2_2}{4 \varepsilon}\right),
\end{equation}
where $\mathbf{X}_i$ and $\mathbf{X}_j$ are data points in the ambient Euclidean space, and $\varepsilon$ is the length-scale parameter. It could be tempting to simply substitute the Euclidean norm in the definition of the Gaussian kernel by a metric on the Grassmann manifold. However, this procedure yields a non positive semi-definite, although symmetric, kernel \cite{harandi2014}. 

Several families of Grassmannian kernels, with different characteristics, are proposed in the literature (see \cite{hamm2008, hamm2009, harandi2014}); however, the most popular positive semi-definite kernels are the Binet-Cauchy and Projection kernels corresponding to the Pl{\"u}cker and projection embeddings, respectively as discussed next. 

% \subsection{Embedding the Grassmann manifold}
% \label{S:4.1}
% \textcolor{red}{The construction of both Binet-Cauchy and projection kernels are discussed in this section based on the Pl{\"u}cker and projection embeddings, respectively. The analysis of the Binet-Cauchy kernel is included herein as a mean to propitiate a rigorous choice of the most appropriate kernel for the applications presented in this paper.}

\subsection{Binet-Cauchy kernel: Pl{\"u}cker embedding}
\label{S:4.1.1}
The construction of the Binet-Cauchy kernel is based on the Pl{\"u}cker embedding $P: \mathcal{G}(p,n) \rightarrow \mathbb{P}\left(\bigwedge^{\raisebox{-0.4ex}{\scriptsize $p$}} \mathbb{R}^{n}\right)$. It embeds the Grassmann manifold $\mathcal{G}(p,n)$ into the projective space $\mathbb{P}\left(\bigwedge^{\raisebox{-0.4ex}{\scriptsize $p$}} \mathbb{R}^{n}\right)$, where the exterior product $\bigwedge^{\raisebox{-0.4ex}{\scriptsize $p$}} \mathbf{V}$ is the k-$th$ product of a vector space $\mathbf{V}$. This property is used to determine an inner product over $\mathbb{P}\left(\bigwedge^{\raisebox{-0.4ex}{\scriptsize $p$}} \mathbb{R}^{n}\right)$ using the $q$-th compound matrix $C_q(\mathbf{\Psi})$ of a matrix $\mathbf{\Psi}$, which are the minors of $\mathbf{\Psi}$ of order $q$ arranged in a lexicographic order \cite{harandi2014}. Next, considering the Binet-Cauchy theorem \cite{vishwanathan2004,harandi2014}, and the points $\mathcal{X}_i = \mathrm{span}\left(\mathbf{\Psi}_i\right) \in \mathcal{G}(n,k)$ and $\mathcal{X}_j = \mathrm{span}\left(\mathbf{\Psi}_j\right) \in \mathcal{G}(n,l)$, one can use the $q$-th compound matrix $C_q(\mathbf{\Psi}_i^T\mathbf{\Psi}_j) = C_q(\mathbf{\Psi}_i)^T C_q(\mathbf{\Psi}_j)$ to determine the inner product for the Pl{\"u}cker embedding as $k_{bc}(\mathbf{\Psi}_i,\mathbf{\Psi}_j) = \mathrm{Tr}\left[C_q(\mathbf{\Psi}_i)^T C_q(\mathbf{\Psi}_j)\right] = \mathrm{Tr}\left[C_q(\mathbf{\Psi}_i^T\mathbf{\Psi}_j)\right] = \mathrm{det}\left(\mathbf{\Psi}_i^T\mathbf{\Psi}_j\right)$. However, the sign of $\mathrm{det}(\cdot)$ can change when columns of $\mathbf{\Psi}_0$ are permuted, a problem circumvented by taking the square value of $\mathrm{det}\left(\mathbf{\Psi}_i^T\mathbf{\Psi}_j\right)$. Thus, one can write $k_{bc}(\mathbf{\Psi}_i,\mathbf{\Psi}_j) = \mathrm{det}\left(\mathbf{\Psi}_i^T\mathbf{\Psi}_j\right)^2$. In fact, several families of kernels are constructed based on the Binet-Cauchy theorem. It includes the polynomial generalizations $k_{bc}(\mathbf{\Psi}_i,\mathbf{\Psi}_j) = \left[\beta + \mathrm{det}\left(\mathbf{\Psi}_i^T\mathbf{\Psi}_j\right)\right]^\alpha$ (see \cite{harandi2014} for a detailed description). Herein, the adopted definition of the Binet-Cauchy kernel is given by 
\begin{equation}\label{eq:4.2}
    k_{bc}(\mathbf{\Psi}_i,\mathbf{\Psi}_j) = \mathrm{det}\left(\mathbf{\Psi}_i^T\mathbf{\Psi}_j\right)^2,
\end{equation}
\noindent
whose relation with the principal angles is given by \cite{hamm2008,harandi2014}
\begin{equation}\label{eq:4.3}
    k_{bc}(\mathbf{\Psi}_i,\mathbf{\Psi}_j) = \prod_{k=1}^{p}\mathrm{cos}^2(\theta_k).
\end{equation}

\subsection{Projection kernel: projection embedding}
\label{S:4.1.2}
The projection kernel is defined straightforwardly using the projection embedding $\Pi: \mathcal{G}(p,n) \rightarrow \mathbb{R}^{n \times n}$, which is given by $\Pi \left(\mathbf{\Psi}\right) = \mathbf{\Psi}\mathbf{\Psi}^T$. This map corresponds to a diffeomorphism, a differentiable mapping with a continuous differentiable inverse, from the Grassmann manifold to the set of rank $p$ symmetric orthogonal projection matrices. Therefore, a natural definition of inner product for the projection embedding is given by $\langle \mathbf{\Psi}_i, \mathbf{\Psi}_j \rangle_{\Pi} = \mathrm{Tr}\left[\Pi \left(\mathbf{\Psi}_i\right)^T \Pi\left(\mathbf{\Psi}_j\right) \right] = ||\mathbf{\Psi}_i^T\mathbf{\Psi}_j||_F^2$. As for the Binet-Cauchy kernel, several families can be obtained using the projection embedding, however, herein the projection kernel is defined as
\begin{equation}\label{eq:4.4}
    k_{pr}(\mathbf{\Psi}_i,\mathbf{\Psi}_j) = ||\mathbf{\Psi}_i^T\mathbf{\Psi}_j||_F^2,
\end{equation}
\noindent
whose relation with the principal angles is given by \cite{hamm2008,harandi2014}
\begin{equation}\label{eq:4.5}
    k_{pr}(\mathbf{\Psi}_i,\mathbf{\Psi}_j) = \sum_{k=1}^{p}\mathrm{cos}^2(\theta_k).
\end{equation}

\subsection{Kernel selection}
Selecting an appropriate kernel requires a detailed analysis of the effect of the dimensionality of the Grassmann manifold on the magnitude of the kernel matrix entries. This analysis is presented for the projection and Binet-Cauchy kernels in Appendix \ref{a:grassmannian_kernel_dimensionality}.  The lemmas introduced in Appendix \ref{a:grassmannian_kernel_dimensionality} are important for the selection of the most appropriate kernel because it is demonstrated that as $p$ increases up to $n/2$, the Binet-Cauchy kernel tends to the identity matrix, which is undesirable for the applications in this paper. On the contrary, the off-diagonal entries of projection kernel grow like $p^2$, making it less sensitive to changes on the dimensionality of the Grassmann manifold. Thus, the projection kernel is employed throughout this work. 

\section{Grassmannian diffusion maps based data classification}
\label{S:6}

In this section, a novel data classification technique considering an overcomplete dictionary of the Grassmannian diffusion coordinates is presented. This technique uses the sparse representation to determine a linear and compact combination of the elements in a dictionary of testing data. It has the advantage that the data's intrinsic geometry is explicitly considered when points on the Grassmann manifold are embedded into a low-dimensional Euclidean space. Moreover, the dimensionality of the obtained linear system is reduced in comparison to conventional approaches \cite{wright2009}. Therefore, one can try to find an appropriate sparse representation of a given object by solving a convex optimization problem in the $l_0$-norm sense; however, this problem is NP-hard. On the other hand, solving it in the $l_2$-norm sense may not yield a sparse solution. To circumvent this limitation, the optimization problem is relaxed and solved in the $l_1$-norm sense \cite{candes2008,wright2009}.

Consider a test sample $\mathbf{X}_T$ and a training set $T_N = \left\{ \mathbf{X}_i\right\}^N_{i=1}$, where each sample in $T_N$ is assigned to a class $k$. We begin by constructing an extended dataset $T_{N+1} = \left\{ \mathbf{X}_1, \dots, \mathbf{X}_N, \mathbf{X}_T\right\}$. Using GDMaps on $T_{N+1}$, we obtain the Grassmannian diffusion coordinates $\mathbf{\Xi}_{i}$ corresponding to the elements $\mathbf{X}_i$. Next, the diffusion coordinates of each element in the $k$-th class are arranged as the columns of a matrix $\mathbf{A}_k \doteq \left[\mathbf{\Xi}_{k,1}, \dots, \mathbf{\Xi}_{k,N_k} \right] \in \mathbb{R}^{q \times N_k}$, where $N_k$ is the number of elements in class $k$, and $q$ is the truncation index used to define the dimension of the diffusion space. This approach assumes that any test sample given by its diffusion coordinates $\mathbf{\Xi}_T$ will lie in the linear span of the training samples in a given class such that
\begin{equation}\label{eq:6.1}
    \mathbf{\Xi}_T  = c_{k,1}\mathbf{\Xi}_{k,1} + \dots + c_{k,N_k}\mathbf{\Xi}_{k,N_k},
\end{equation}
\noindent
where $c_{k,j} \in \mathbb{R}$ with $j=1, \dots, N_k$. Therefore, the dictionary $\mathbf{A}$ can be constructed by concatenating the matrices $\mathbf{A}_k$, such that
\begin{equation}\label{eq:6.2}
    \mathbf{A} \doteq \left[\mathbf{A}_1, \dots, \mathbf{A}_{N_k} \right] \in \mathbb{R}^{q \times N}.
\end{equation}
\noindent
yielding the following underdetermined ($N > q$) linear system:
\begin{equation}\label{eq:6.3}
    \mathbf{\Xi}_T = \mathbf{A}\mathbf{c},
\end{equation}
\noindent
where the constant vector $\mathbf{c} = [0, \dots, 0, c_{k,1}, \dots, c_{k,N_k}, 0, \dots, 0]^T \in \mathbb{R}^{N}$ has non-zero entries in the positions associated with the $k$-th class. It is evident that, in general, no unique solution, $\hat{\mathbf{c}}$, exists for this underdetermined system; however, if $\hat{\mathbf{c}}$ is sufficiently sparse, the exact solution can be recovered with high probability \cite{wright2009}. Moreover, real-life data can be corrupted with noise; thus, we seek an approximate solution of the linear system in Eq. (\ref{eq:6.3}) by solving the following optimization problem
\begin{equation}\label{eq:6.5}
    \hat{\mathbf{c}} = \mathrm{arg}\:\mathrm{min}||\mathbf{c}||_1 \quad \text{subject to} \quad ||\mathbf{A}\mathbf{c} - \mathbf{\Xi}_T||^2_2 \leq \epsilon,
\end{equation}
\noindent
or its unconstrained form,
\begin{equation}\label{eq:6.6}
    \hat{\mathbf{c}} = \mathrm{arg}\:\mathrm{min}  ||\mathbf{A}\mathbf{c} - \mathbf{\Xi}_T||^2_2 + \beta ||\mathbf{c}||_1,
\end{equation}
\noindent
where $\epsilon$ is the error tolerance and $\beta$ is the regularization constant. With the vector of coefficients $\hat{\mathbf{c}}$, one can perform the classification either by assigning the maximum entry of $\hat{\mathbf{c}}$ to its associated class or by identifying the class yielding the smallest error in the approximation \cite{wright2009}. Clearly, different methods can be employed when performing this task. We propose to use the subspace structure in the Grassmannian diffusion space to compute a residual as
\begin{equation}\label{eq:6.7}
    r(k) = ||\mathbf{A}\left(\mathbf{I}^{*}_{k} \circ \hat{\mathbf{c}}_k\right) - \mathbf{\Xi}_T||_2.
\end{equation}
\noindent
The classification is performed by finding $k$ that minimizes $r(k)$, where the entries of $\mathbf{I}^{*}_{k} \in \mathbb{R}^{q}$ corresponding to the $k$-th class are equal to one and the rest are zero, and $\circ$ is the Hadamard product. The mechanization of this classification technique is presented in Algorithm \ref{alg:6.1}.

\begin{algorithm}[h]
\caption{GDMaps based data classification using sparse representation}
\label{alg:6.1}
\begin{algorithmic}[1]
\REQUIRE a set of $N$ high-dimensional data $T_N = \left\{ \mathbf{X}_1, \dots \mathbf{X}_N\right\}$ with $\mathbf{X}_i \in \mathbb{R}^{n \times m}$, the dimension $p$ of the Grassmann manifold, and a test sample $\mathbf{X}_T \in \mathbb{R}^{n \times m}$.
\FOR{$i \in 1, \dots, N$}
\STATE Compute the thin Singular Value Decomposition: $\mathbf{X}_i = \mathbf{\Psi}_i\mathbf{\Sigma}_i\mathbf{\Phi}_i^T$; where, $\mathbf{\Psi}_i \in \mathcal{V}(p,n)$ and $\mathbf{\Phi}_i \in \mathcal{V}(p,m)$.
\ENDFOR
\STATE Compute the thin Singular Value Decomposition: $\mathbf{X}_T = \mathbf{\Psi}_T\mathbf{\Sigma}_T\mathbf{\Phi}_T^T$; where, $\mathbf{\Psi}_T \in \mathcal{V}(p,n)$ and $\mathbf{\Phi}_T \in \mathcal{V}(p,m)$.
\STATE For the augmented sets $\mathbf{\Psi} = \left\{\mathbf{\Psi}_1, \dots, \mathbf{\Psi}_N, \mathbf{\Psi}_T\right\}$ and $\mathbf{\Phi} = \left\{\mathbf{\Phi}_1, \dots, \mathbf{\Phi}_N, \mathbf{\Phi}_T\right\}$ compute the entries $k_{ij}$ of the kernel matrices $k_{ij}\left(\mathbf{\Psi}\right)$ and $k_{ij}\left(\mathbf{\Phi}\right)$ using either Eq. (\ref{eq:4.2}) or Eq. (\ref{eq:4.4}).
\STATE If necessary, compute the composite kernel matrix $k\left(\mathbf{\Psi},\mathbf{\Phi}\right)$. For example:\\
\vspace{0.2cm}
$k\left(\mathbf{\Psi},\mathbf{\Phi}\right) = k_{ij}\left(\mathbf{\Psi}\right)+k_{ij}\left(\mathbf{\Phi}\right)$ or $k\left(\mathbf{\Psi},\mathbf{\Phi}\right) = k_{ij}\left(\mathbf{\Psi}\right) \circ k_{ij}\left(\mathbf{\Phi}\right)$, where $\circ$ is the Hadamard product.
\vspace{0.2cm}
\STATE Compute the diffusion coordinates $\mathbf{\Xi}_1, \dots, \mathbf{\Xi}_N, \mathbf{\Xi}_T$ as in Algorithm \ref{alg:5.1}, with $\mathbf{\Xi}_i \in \mathbb{R}^{q}$.
\STATE Concatenate the $\mathbf{\Xi}_i$'s of each class $k$ as columns of a matrix $\mathbf{A}_k \in \mathbb{R}^{q \times N_k}$.
\STATE Create the matrix of training diffusion coordinates $\mathbf{A} = \left[\mathbf{A}_1, \dots, \mathbf{A}_{N_k} \right] \in \mathbb{R}^{q \times N}$
\STATE Normalize the columns of $\mathbf{A}$ and the test sample $\mathbf{\Xi}_T$ to have unit $l_2$-norm.
\STATE Solve the convex optimization problem:\\
\vspace{0.2cm}
$\hat{\mathbf{c}} = \mathrm{arg}\:\mathrm{min}||\mathbf{c}||_1 \quad \text{subject to} \quad ||\mathbf{A}\mathbf{c} - \mathbf{\Xi}_T||^2_2 \leq \epsilon$\\
\vspace{0.2cm}
or alternatively\\
\vspace{0.2cm}
$\hat{\mathbf{c}} = \mathrm{arg}\:\mathrm{min}  ||\mathbf{A}\mathbf{c} - \mathbf{\Xi}_T||^2_2 + \beta ||\mathbf{c}||_1$
\vspace{0.2cm}
\STATE Compute the residuals for each class $k$:
\vspace{0.2cm}
$r(k) = ||\mathbf{A}\left(\mathbf{I}^{*}_{k} \circ \hat{\mathbf{c}}_k\right)- \mathbf{\Xi}_T||_2$
\vspace{0.2cm}
\ENSURE $k\left(\mathbf{X}_T\right) = \mathrm{arg}\:\underset{k}{\mathrm{min}}\ r(k)$
\end{algorithmic}
\end{algorithm}

\section{Examples}
\label{S:7}
In this section, three problems are considered to assess the performance of the Grassmannian diffusion maps in revealing the intrinsic geometric structure of a dataset for classification purposes. Herein, the projection kernel (Eq. (\ref{eq:4.4})) is adopted because it is not as sensitive as the Binet-Cauchy kernel to changes in the dimensionality of the Grassmann manifold, as demonstrated in the Appendix \ref{a:grassmannian_kernel_dimensionality}. In the ensuing analysis, both the conventional and Grassmannian diffusion maps are considered for comparison. The algorithms of both techniques are implemented in the UQpy software \cite{olivier2020uqpy}(Uncertainty Quantification with python: \url{https://uqpyproject.readthedocs.io/en/latest/}), a general purpose Python toolbox for modeling uncertainty in physical and mathematical systems.

\subsection{Structured data on the unit sphere in $\mathbb{R}^3$}
\label{S:7.1}
Consider a collection $S_N = \{\mathbf{X}_i\}_{i=1}^N$ of $N=3,000$ points defining two cone-like structures in $\mathbb{R}^3$ as presented in Fig. \ref{fig:4}a.  In this problem, every point in the set $S_N$ is composed of three coordinates $(x_1, x_2, x_3)$ that lie on a surface constrained by the relation $\mathrm{sin}(\phi) = \mathrm{cos}^2(\theta)$ in spherical coordinates with radius $r$ uniformly distributed in the interval $[0,2]$. The set of random points on this constrained manifold are shown in Fig. \ref{fig:4}a, where the color map is determined by the magnitude of each point $\sqrt{x^2_1 + x^2_2 + x^2_3}$. By imposing these constraints, we can see (Figure \ref{fig:4}b) that the points are constrained to lie in subspaces that form a one-dimensional curve on $\mathcal{G}(1,3)$, i.e. the unit sphere in $\mathbb{R}^3$.

This example demonstrates that the Grassmannian diffusion maps can identify the intrinsic structure of points on $\mathcal{G}(1,3)$ given by the subspaces spanned by $\mathbf{\Psi} \in \mathcal{V}(1,3)$ (Fig.\ \ref{fig:4}b). As $\mathcal{G}(1,3)$ is represented by the unit sphere in $\mathbb{R}^3$, the affinity between points on it is given by the principal angle between pairs of unit vectors in $\mathbb{R}^3$.
\begin{figure}[tbhp]
	\centering
	\captionsetup{justification=centering}
	\includegraphics[width=0.9\textwidth]{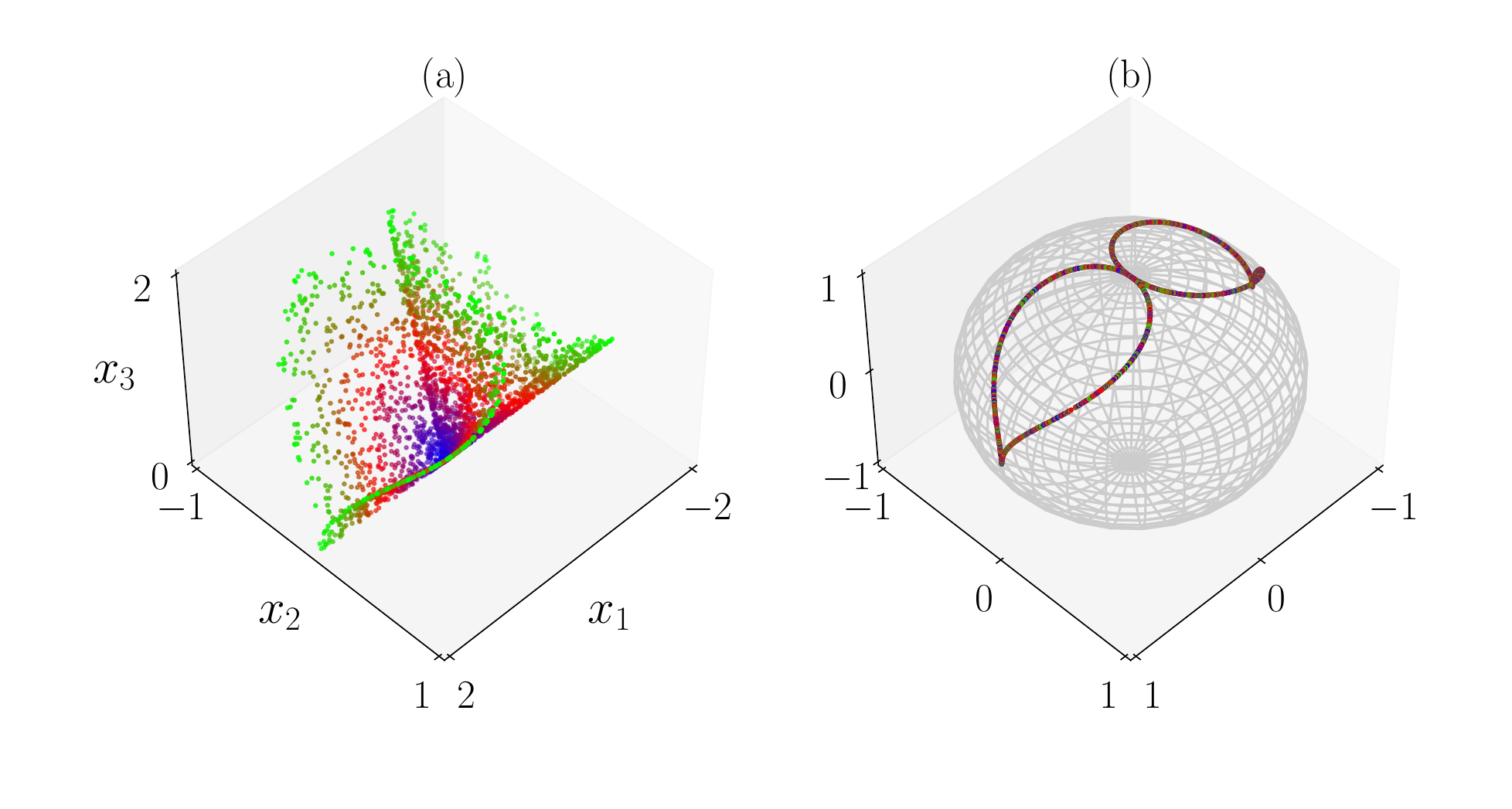} %.25
	%\vspace{-1.5em}
	\caption{Example 1: a) Elements of $S_N$ in $\mathbb{R}^3$ colored by their magnitude. b) Projection of the elements of $S_N$ onto the Grassmann manifold colored by their magnitude in the ambient space.}
	\vspace{-0.5em}
	\label{fig:4}
\end{figure}

Figs.\ \ref{fig:6}a and \ref{fig:6}b show the diffusion coordinates for the conventional diffusion maps with a color map determined by the first and the second diffusion coordinates, respectively. On the other hand, Fig. \ref{fig:6}c and \ref{fig:6}d show the diffusion coordinates for the Grassmannian diffusion maps with a color map determined by the first and the second diffusion coordinates, respectively. One can easily see that a radial symmetry is identified by the Grassmannian diffusion maps. This is an expected outcome due to the geometry of the data in the ambient space. 

\begin{figure}[tbhp]
	\centering
	\captionsetup{justification=centering}
	\includegraphics[width=0.8\textwidth]{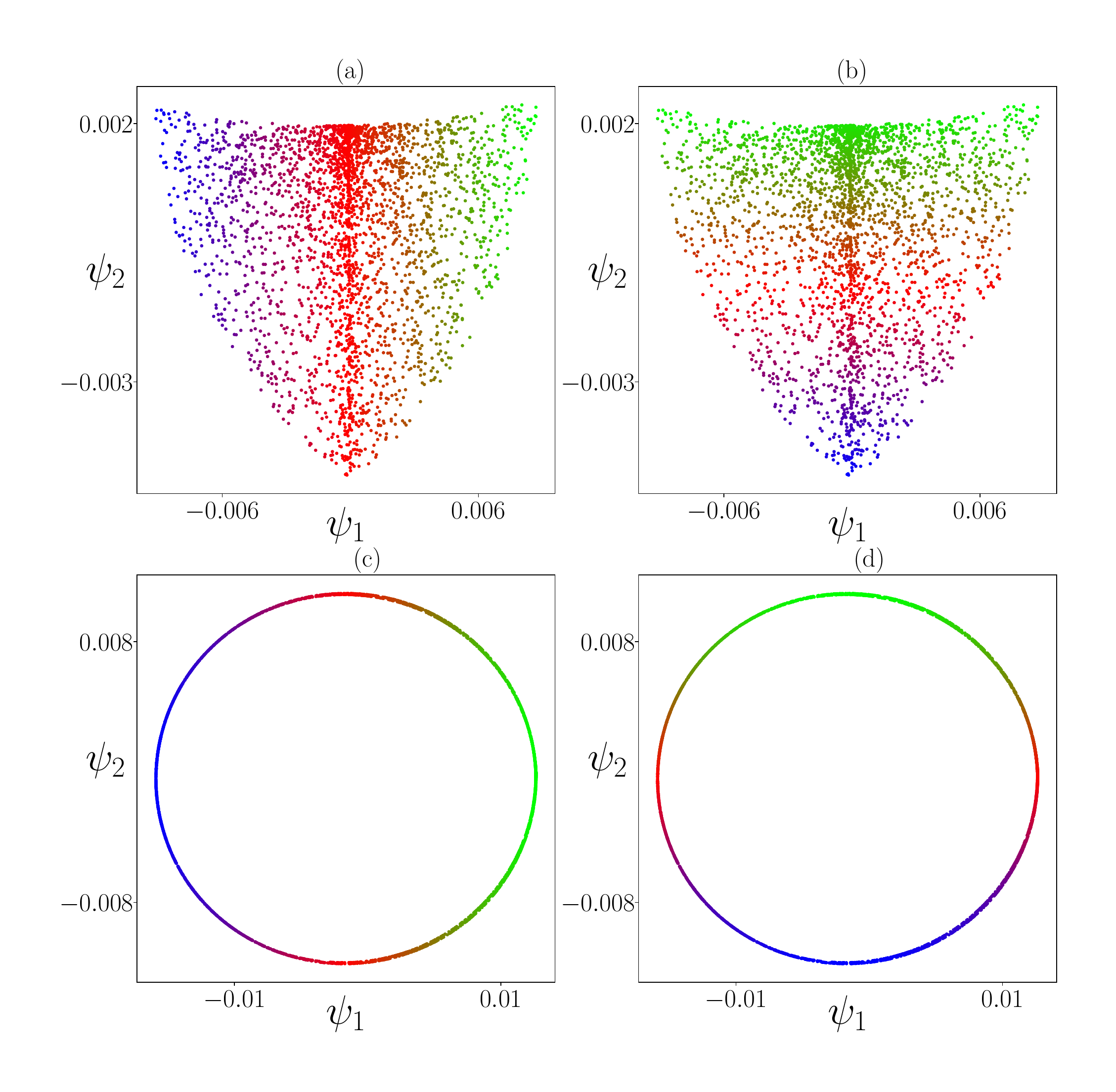}
	\vspace{-0.5em}
	\caption{{\color{red}Example 1: }Diffusion coordinates: a) {\color{red}Conventional DMaps} with color map defined by $\psi_1$ and b) $\psi_2$, c) {\color{red}Grassmannian DMaps} with color map defined by $\psi_1$ and by d) $\psi_2$.}
	\vspace{-0.5em}
	\label{fig:6}
\end{figure}

The performance of both conventional diffusion maps and the Grassmannian diffusion maps are assessed in Figs. \ref{fig:7} and \ref{fig:8}, respectively. In both cases, the first two nontrivial diffusion coordinates define the color maps. For the conventional diffusion maps, both the first and the second diffusion coordinates are mapped in the ambient space $\mathbb{R}^3$ according to the canonical directions, as observed in Figs \ref{fig:7}a,c. However, when the same diffusion coordinates are mapped on the Grassmann manifold $\mathcal{G}(1,3)$ they appear shuffled, and no logical parametrization can be extracted from it, as observed in Figs \ref{fig:7}b,d. This means that the conventional DMaps successfully identifies a low-dimensional representation of the data in the ambient space, but this low-dimensional representation does not reflect the underlying subspace structure of the data as it is constrained on the manifold.

On the other hand, for the Grassmannian diffusion maps, the first and the second diffusion coordinates in the ambient space $\mathbb{R}^3$ (Figs. \ref{fig:8}a,c) and on the Grassmann manifold $\mathcal{G}(1,3)$ (Figs. \ref{fig:8}b,d) are clearly structured to align with the intrinsic subspace structure of the data on $\mathcal{G}(1,3)$.
\begin{figure}[tbhp]
	\centering
	\captionsetup{justification=centering}
	\includegraphics[width=0.8\textwidth]{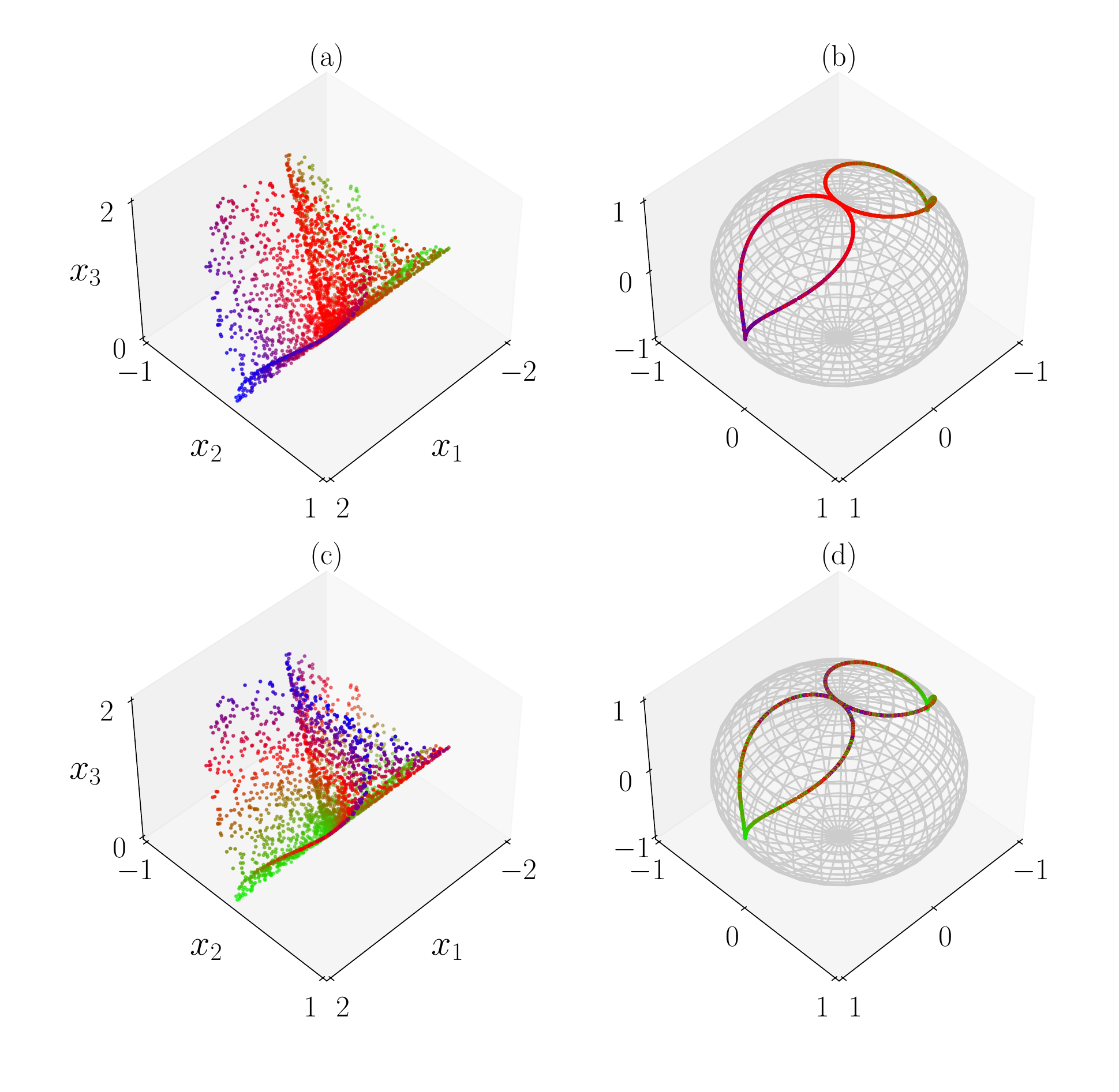}
	\vspace{-0.5em}
	\caption{Example 1: Conventional diffusion maps: a) color map for $\psi_1$ in the ambient space, and b) on the Grassmann manifold. c) color map for $\psi_2$ in the ambient space, and d) on the Grassmann manifold.}
	\vspace{-0.5em}
	\label{fig:7}
\end{figure}
\begin{figure}[tbhp]
	\centering
	\captionsetup{justification=centering}
	\includegraphics[width=0.8\textwidth]{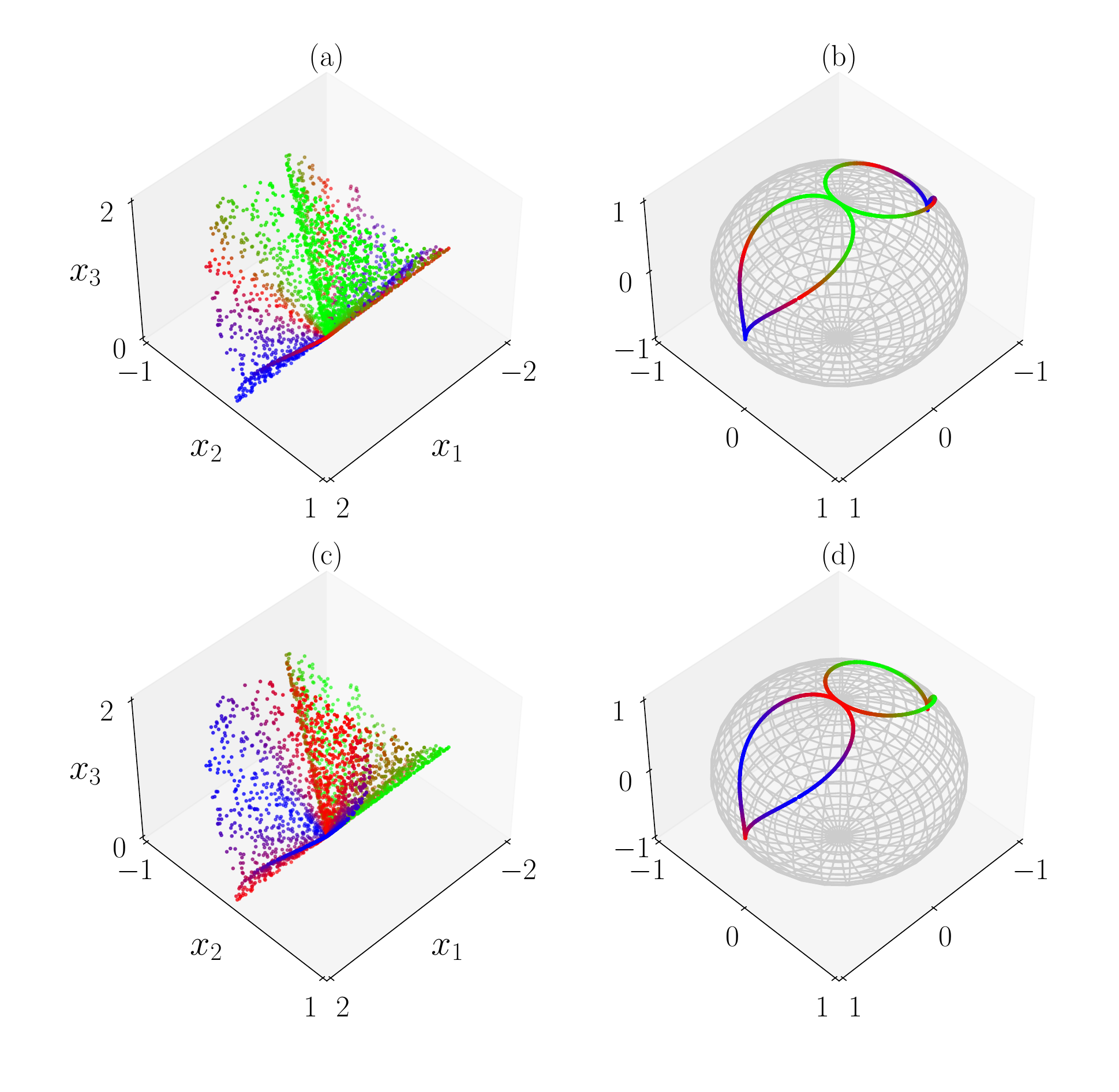}
	\vspace{-0.5em}
	\caption{Example 1: Grassmannian diffusion maps: a) color map for $\psi_1$ in the ambient space, and b) on the Grassmann manifold. c) color map for $\psi_2$ in the ambient space, and d) on the Grassmann manifold.}
	\vspace{-0.5em}
	\label{fig:8}
\end{figure}

\subsection{Robust classification of high-dimensional random field data}
\label{S:7.2}
In this example, the performance of the Grassmannian diffusion maps in the classification of high-dimensional random field data is investigated. We employ k-means on the obtained diffusion coordinates in order to cluster the data according to the most important features in the ambient space. 

Consider the elements of a dataset $S_N = \left\{\mathbf{X}_1, \dots, \mathbf{X}_N \right\}$, with $N=3,000$, generated with a prescribed rank $p=5$, where $\mathbf{X}_k \in \mathbb{R}^{n \times m}$, with $n=m=40$, is given by
\begin{equation}\label{eq:6.1}
    \mathbf{X}_k = \mathbf{U}_k\mathbf{A}_k\mathbf{U}_k^T,
\end{equation}
\noindent
where $\mathbf{A}_k \in \mathbb{R}^{p \times p}$ is a diagonal matrix whose the elements are i.i.d. random numbers with uniform distribution in the interval $(0,1]$, and the entries of $\mathbf{U}_k \in \mathbb{R}^{n \times p}$ take the following functional form
\begin{equation}\label{eq:6.2}
    u^{(k)}_{ij}=\sqrt{\frac{2}{n}}\mathrm{cos}\left[\frac{2 \pi (j+L_k)}{n}\left(i - T_k\right) \right],
\end{equation}
\noindent
with $j = 0, \dots, p$, and $i=0, \dots, n-1$. Further, $T_k$ and $L_k$ are uniform discrete random variables assuming integer values in the interval $[0, n-1]$ and $[1, \lfloor \frac{n}{2} \rfloor + 1 - p]$, respectively. Note that the samples $T_k$ and $L_k$ are the same for every column of $\mathbf{U}_k$, meaning that each element of $S_N$ is a realization of a two-dimensional random field having a stochastic basis. Our objective is thus to cluster these random fields according to their stochastic basis. Fig.\ \ref{fig:10} shows two realizations of $\mathbf{X}$ where $\mathbf{A}_1 = \mathrm{diag}(0.444, 0.828, 0.775,  0.913, 0.981)$, $L_1 = 20$, and $T_1 = 3$; and $\mathbf{A}_2 = \mathrm{diag}(0.614, 0.800, 0.184,  0.519, 0.961)$, $L_2 = 38$, and $T_2 = 6$.
\begin{figure}[tbhp]
	\centering
	\captionsetup{justification=centering}
	\includegraphics[width=0.8\textwidth]{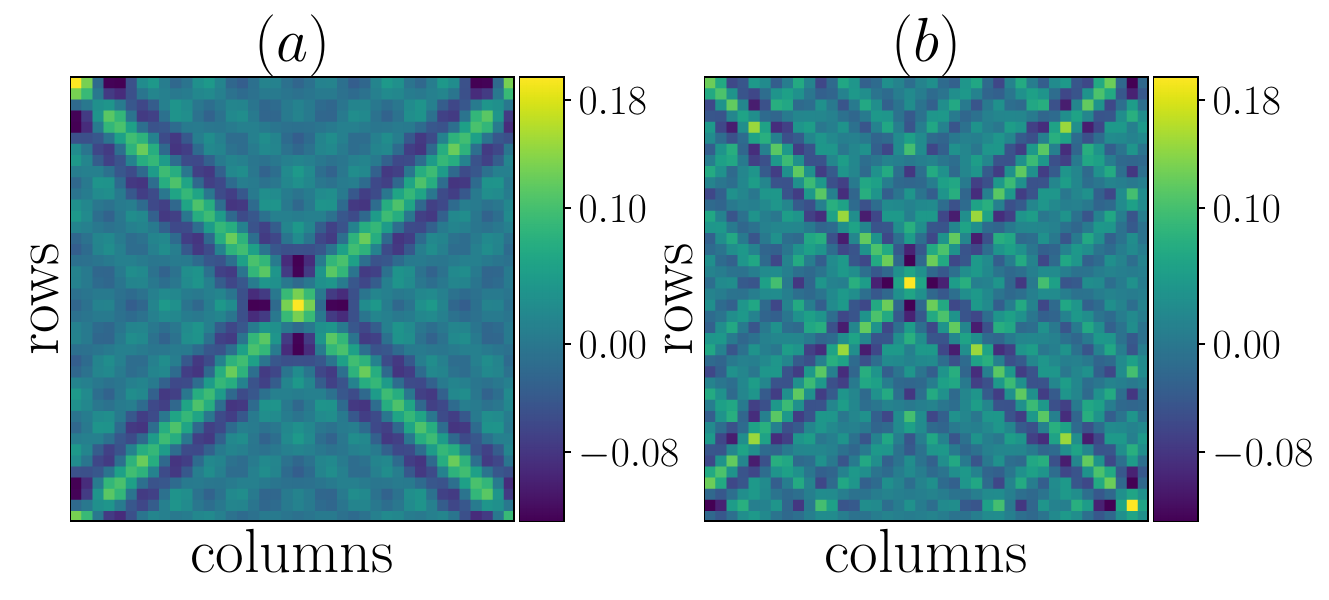}
	%\vspace{-1.5em}
	\caption{{\color{red}Example 2: }Two realizations of random field elements in $S_N$.}
	\vspace{-0.5em}
	\label{fig:10}
\end{figure}
%==============================================

The conventional DMaps and the GDMaps techniques are compared by their ability to reveal the underlying subspace structure of the elements in the set $S_N$.  To explore this, the first three diffusion coordinates obtained using the conventional and the Grassmannian diffusion maps are presented in Fig. \ref{fig:13}a and in Fig. \ref{fig:13}b, respectively; where every point defined by the diffusion coordinates corresponds to an element in $S_N$. In the k-means, we adopt 15 clusters corresponding to the 15 distinct values of $L_k$ in the integer-valued distribution since, from Eq.\ (\ref{eq:6.2}), we would expect $L_k$ to have a strong influence on the shape of the points projected onto $\mathcal{G}(5,40)$. We see that the conventional diffusion coordinates are highly dispersed and their dispersion adversely influences the clustering. This becomes clearer when the clusters are mapped back to the parameter space $\left(T_k \times L_k\right)$ in Fig. \ref{fig:14}a, where the shuffled colors indicate that the classification does not segregate based on the subspace characteristics of the elements of $S_N$ defined by $T_k$  and $L_k$. 

On the other hand, Fig. (\ref{fig:13}b) shows that the Grassmannian diffusion maps reveal clearly delineated clusters  identifying the subspace geometry of the elements in $S_N$. This behavior is evident in Fig. \ref{fig:14}b, where the clusters are mapped back onto the parameter space $\left(T_k \times L_k\right)$ and identify the elements belonging to the subsets defined by each $L_k$ with $k=1, \dots, 15$.
\begin{figure}[tbhp]
	\centering
	\captionsetup{justification=centering}
	\includegraphics[width=0.8\textwidth]{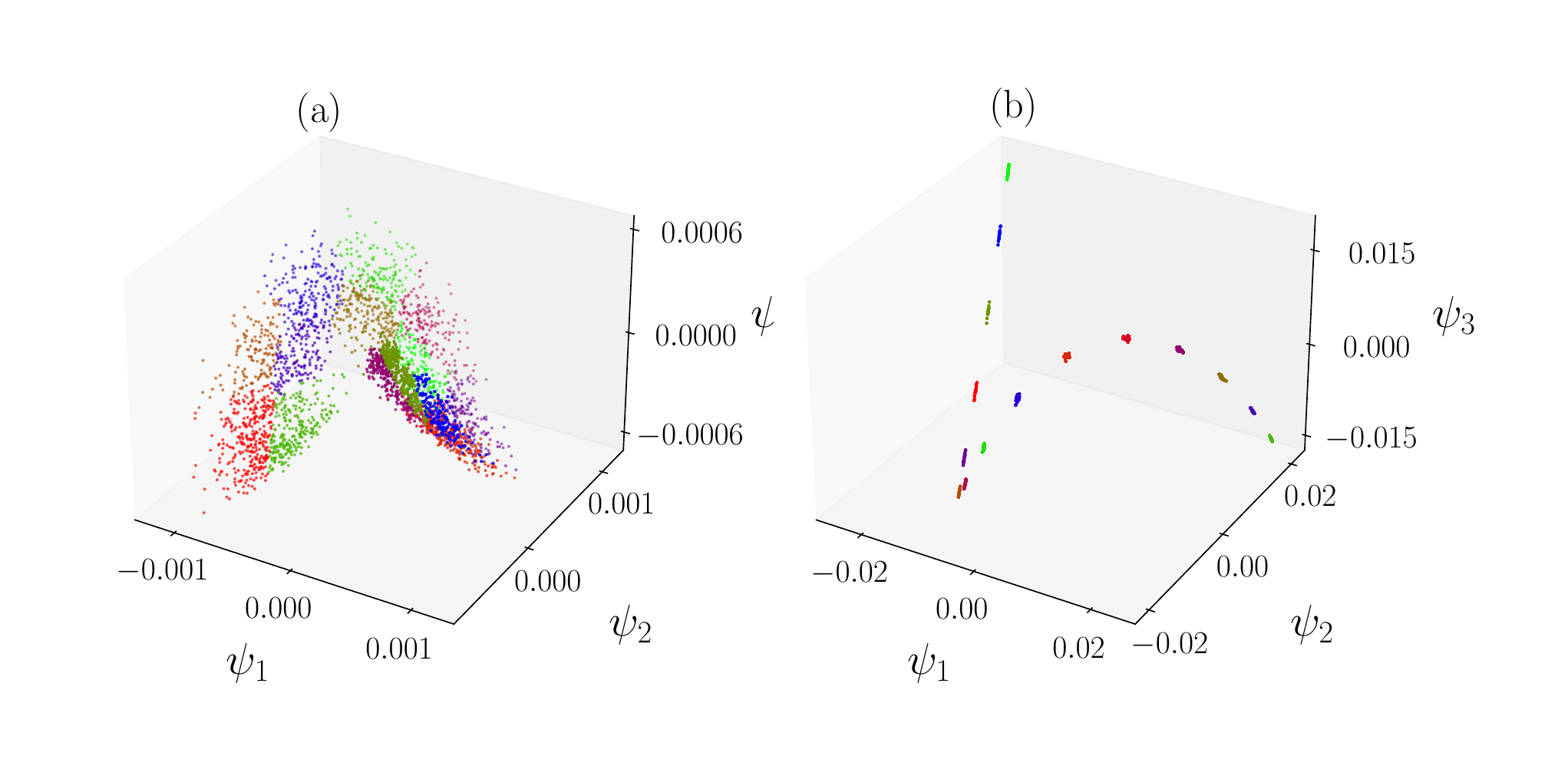}
	\vspace{-0.5em}
	\caption{Example 2: k-means applied on the first three diffusion coordinates (15 clusters): a) Conventional and b) Grassmannian diffusion maps: }
	\vspace{-0.5em}
	\label{fig:13}
\end{figure}

\begin{figure}[tbhp]
	\centering
	\captionsetup{justification=centering}
	\includegraphics[width=0.8\textwidth]{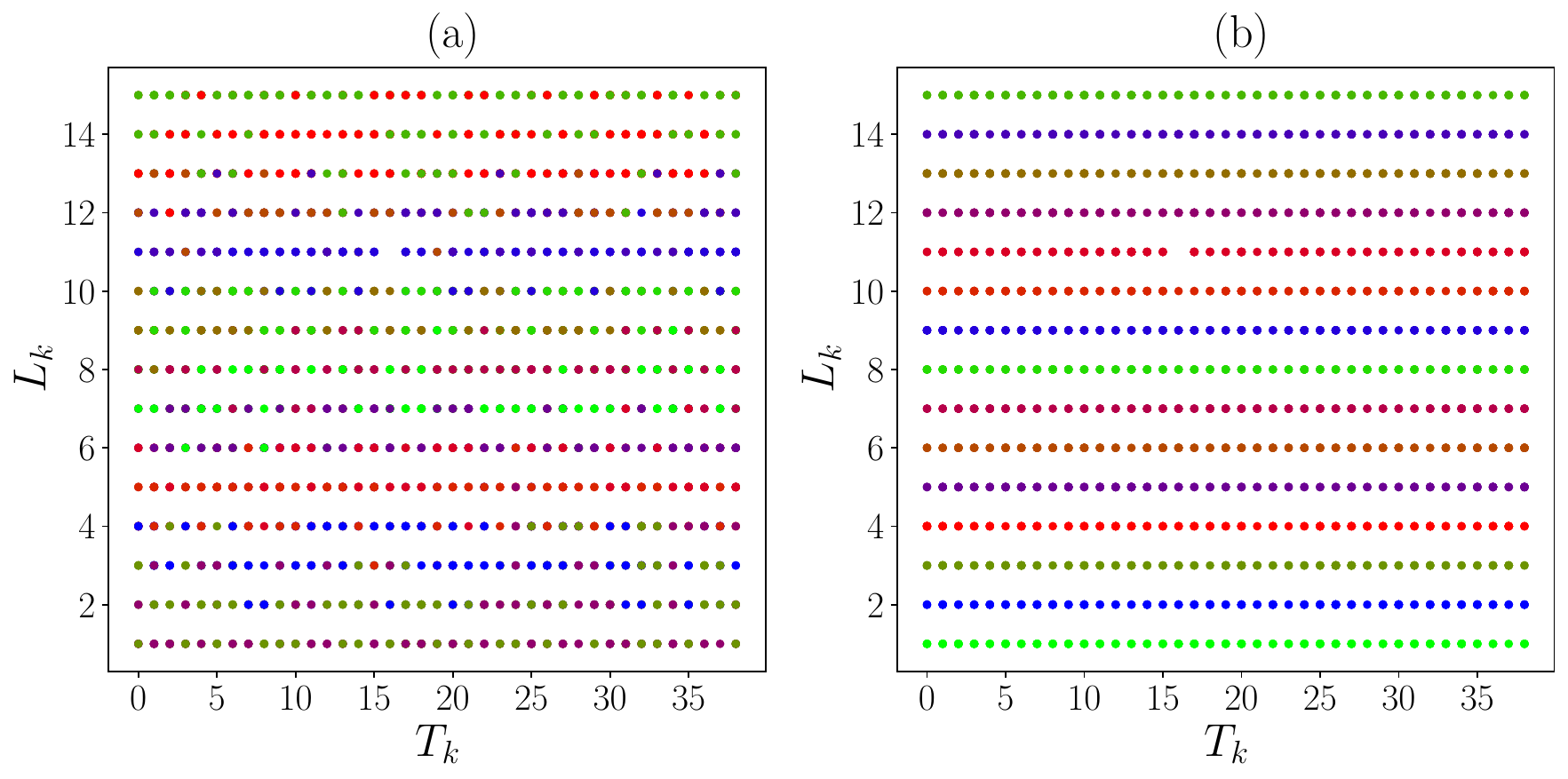}
	%\vspace{-1.5em}
	\caption{Example 2: Clusters of the diffusion coordinates mapped to the parameter space $(T_k \times L_k)$ (a) Conventional diffusion maps (b) Grassmannian diffusion maps.}
	\vspace{-0.5em}
	\label{fig:14}
\end{figure}

\subsection{Sparse representation-based face recognition}
\label{S:7.3}
In this example, faces of different subjects are identified using the sparse representation based classification technique of section \ref{S:6}. The face images used in this experiment were retrieved from the AT\&T Database of Faces (AT\&T Laboratories Cambridge). This database contains a set of 400 face images of 40 subjects, corresponding to 10 face images for each subject, with variation in the illumination, changes in the facial expressions (e.g., open and closed eyes, smiling and not smiling faces), and considering occlusions (e.g., glasses). Moreover, they were taken against a dark and homogeneous background and the subjects had some limited freedom for side-to-side movement. All images are in grey scale and were resized to dimension $200 \times 200$. 

In this experiment, Algorithm \ref{alg:6.1} and a modified version substituting steps 1-6 by the Gaussian kernel and conventional DMaps, are used for face recognition. The training set is composed by $N=360$ face images (40 classes with 9 face images of the same subject in each), and the test set is composed of 40 face images of distinct subjects. First, considering one face image from the test set (Fig. \ref{fig:15}b) and assuming that $p=4$ encodes the important features of each face image, we apply both the conventional and the Grassmannian diffusion maps truncating the dimension of the diffusion coordinates to $q=20$.
%\begin{figure}[tbhp]
    %\centering
    %\subfloat[\centering]{\includegraphics[width=10cm]{figures/Ex3/ex3_2.pdf} }%
    %\qquad
    %\subfloat[\centering]{\includegraphics[width=2cm]{figures/Ex3/ex3_1.pdf} }%
    %\caption{\centering Example 3: a) One face image of each subject representing each of the 40 %classes and b) Test face image.}%
%    \label{fig:15}%
%\end{figure}

\begin{figure}[tbhp]
	\centering
	\captionsetup{justification=centering}
	\subfloat[]{
		\includegraphics[width=10cm]{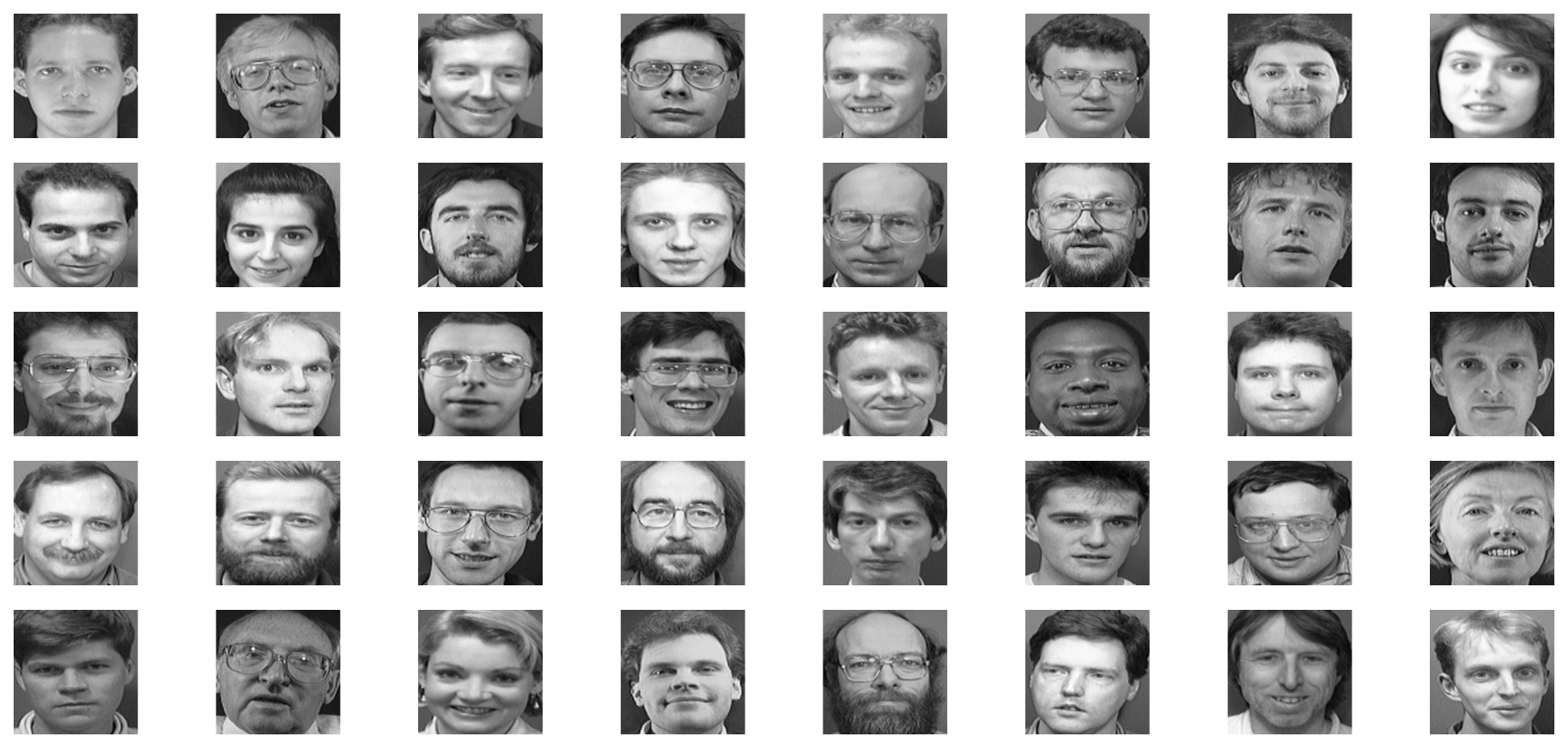}
	}
	\hfill
	\subfloat[]{
		\includegraphics[width=2cm]{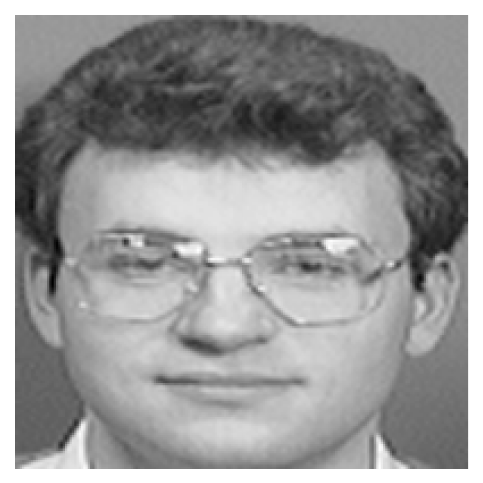}
	}
	\caption{Example 3: a) One face image of each subject representing each of the 40 classes and b) Test face image.}
	\label{fig:15}
\end{figure}

Facial recognition is performed using either the approximated solution $\mathbf{\hat{c}}$ of the underdetermined system or the residuals $r(k)$. Both the constrained and unconstrained convex optimization problems are considered. From Figures \ref{fig:19} and \ref{fig:21} we can clearly see that the present method informed by the conventional diffusion maps cannot recognize the test face image using either $\mathbf{\hat{c}}$ (Figs. \ref{fig:19}) or $r(k)$ (Figs. \ref{fig:21}) as the identification criterion. This observation is valid for both the constrained and unconstrained problems. 

On the other hand, Figs. \ref{fig:23} and Fig. \ref{fig:25} show correct facial recognition using the GDMaps-based method using either $\mathbf{c}$ (Fig. \ref{fig:23}) or $r(k)$ (Fig. \ref{fig:25}) as the identification criterion for both the constrained and unconstrained optimization problems. It is worth noting that the solution $\mathbf{\hat{c}}$ from the constrained problem for the conventional diffusion maps is not compact. Moreover, the condition number of the transition matrix of the Markov process in the conventional diffusion maps is 849 times larger than the one observed in the Grassmannian diffusion maps, which makes the Grassmannian diffusion maps more stable than its conventional counterpart.

\begin{figure}[tbhp]
	\centering
	\captionsetup{justification=centering}
	\includegraphics[width=0.8\textwidth]{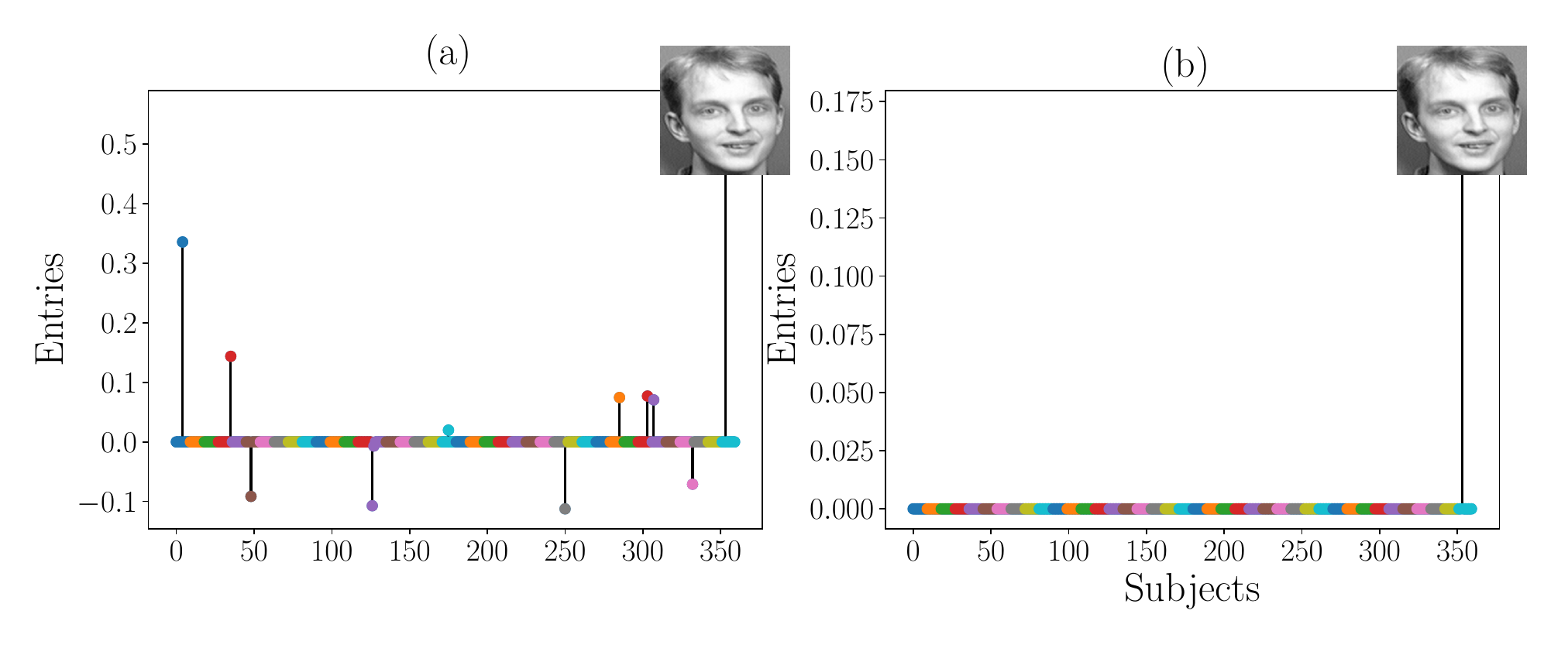}
	%\vspace{-1.5em}
	\caption{Example 3: Conventional diffusion maps -- Entries of $\mathbf{\hat{c}}$, with colors defining the 40 classes, for the a) constrained and b) unconstrained minimization problems.}
	\vspace{-0.5em}
	\label{fig:19}
\end{figure}

\begin{figure}[tbhp]
	\centering
	\captionsetup{justification=centering}
	\includegraphics[width=0.8\textwidth]{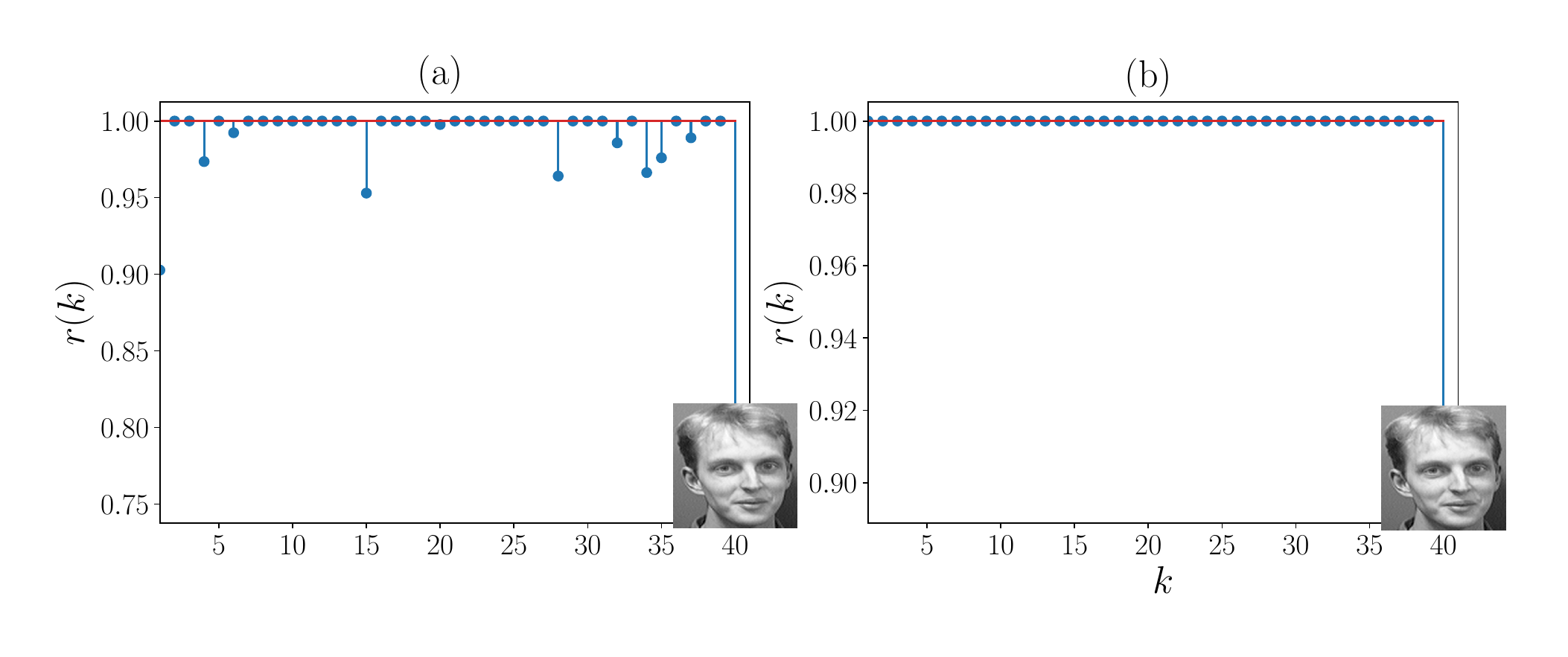}
	%\vspace{-1.5em}
	\caption{Example 3: Conventional diffusion maps -- Residuals $r(k)$ using the a) constrained and b) unconstrained minimization problems.}
	\vspace{-0.5em}
	\label{fig:21}
\end{figure}

\begin{figure}[tbhp]
	\centering
	\captionsetup{justification=centering}
	\includegraphics[width=0.8\textwidth]{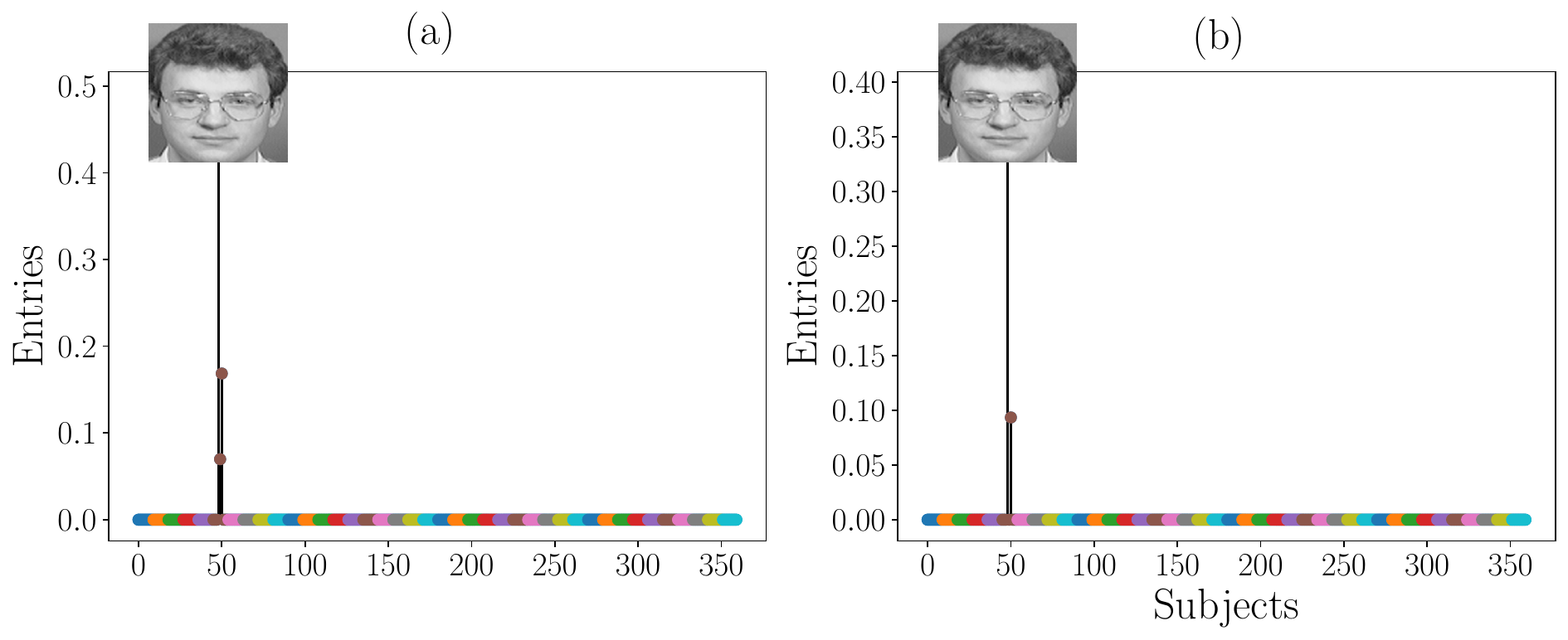}
	%\vspace{-1.5em}
	\caption{Example 3: Grassmannian diffusion maps -- Entries of $\mathbf{c}$, with colors defining the 40 classes, for the a) constrained and b) unconstrained minimization problems.}
	\vspace{-0.5em}
	\label{fig:23}
\end{figure}

\begin{figure}[tbhp]
	\centering
	\captionsetup{justification=centering}
	\includegraphics[width=0.8\textwidth]{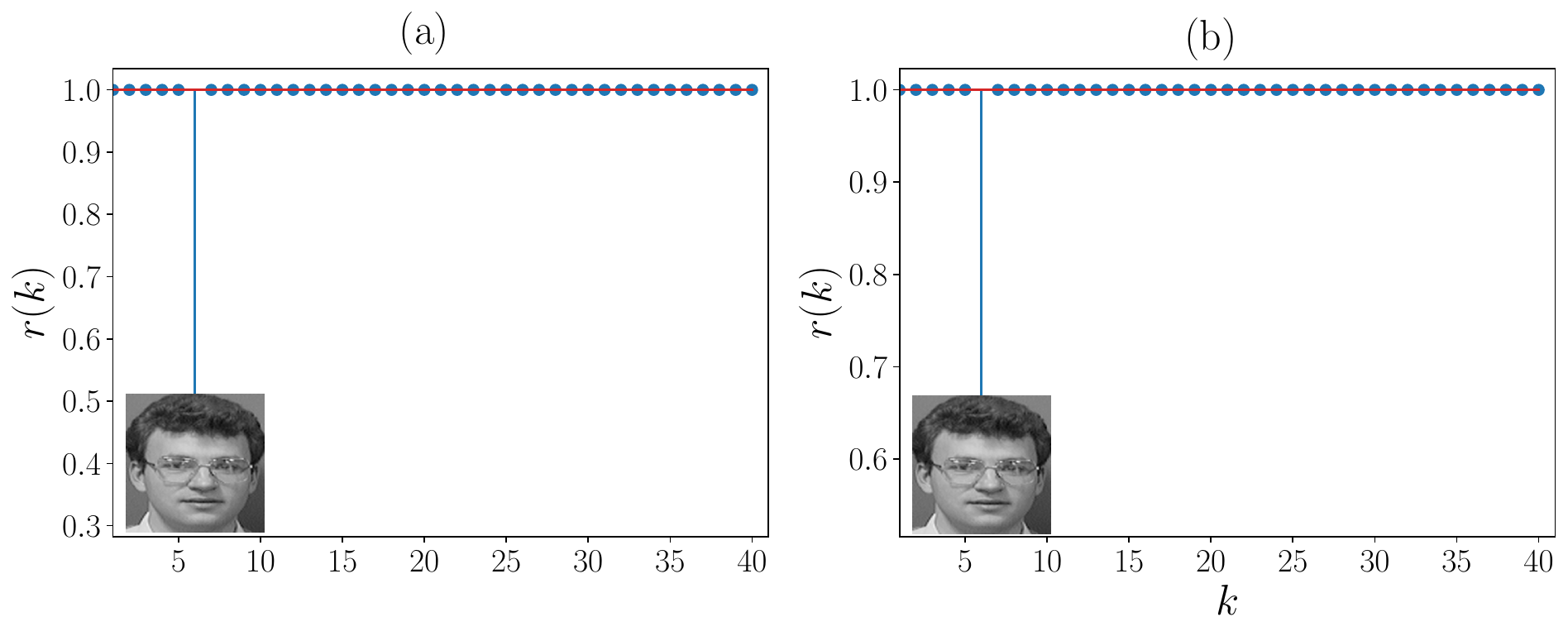}
	%\vspace{-1.5em}
	\caption{Example 3: Grassmannian diffusion maps -- Residuals $r(k)$ using the a) constrained and b) unconstrained minimization problems.}
	\vspace{-0.5em}
	\label{fig:25}
\end{figure}

Finally, the performance of the facial recognition for both the conventional and the Grassmannian diffusion maps is assessed for all 40 subjects in the test set. The approach employing the conventional diffusion maps only correctly identified 2 subjects when the constrained minimization is employed, and just 1 subject when solving the unconstrained minimization problem. This poor performance can be justified by the low condition number of the transition matrix of the Markov process in the conventional approach. 

Figure \ref{fig:27} shows how the recognition rate changes for different values of $p$ in the GDMaps, where we observe that high recognition rates are obtained even considering $p=1$. In the best cases, a recognition rate of $95\%$ was achieved for $p=12, 13$ and $14$. Moreover, we observed that the recognition rate tends to diminish for larger $p$ ($p>50$ for instance); indicating that the rank of the images is low and that higher dimensions are associated with noise and/or non-structural features of the images. Thus, the use of large values of $p$ is not justified in this of application. Moreover, the computational performance of the present technique is comparable with the performance of the existing sparse representation-based classification method \cite{wright2009}.
\begin{figure}[tbhp]
	\centering
	\captionsetup{justification=centering}
	\includegraphics[width=0.6\textwidth]{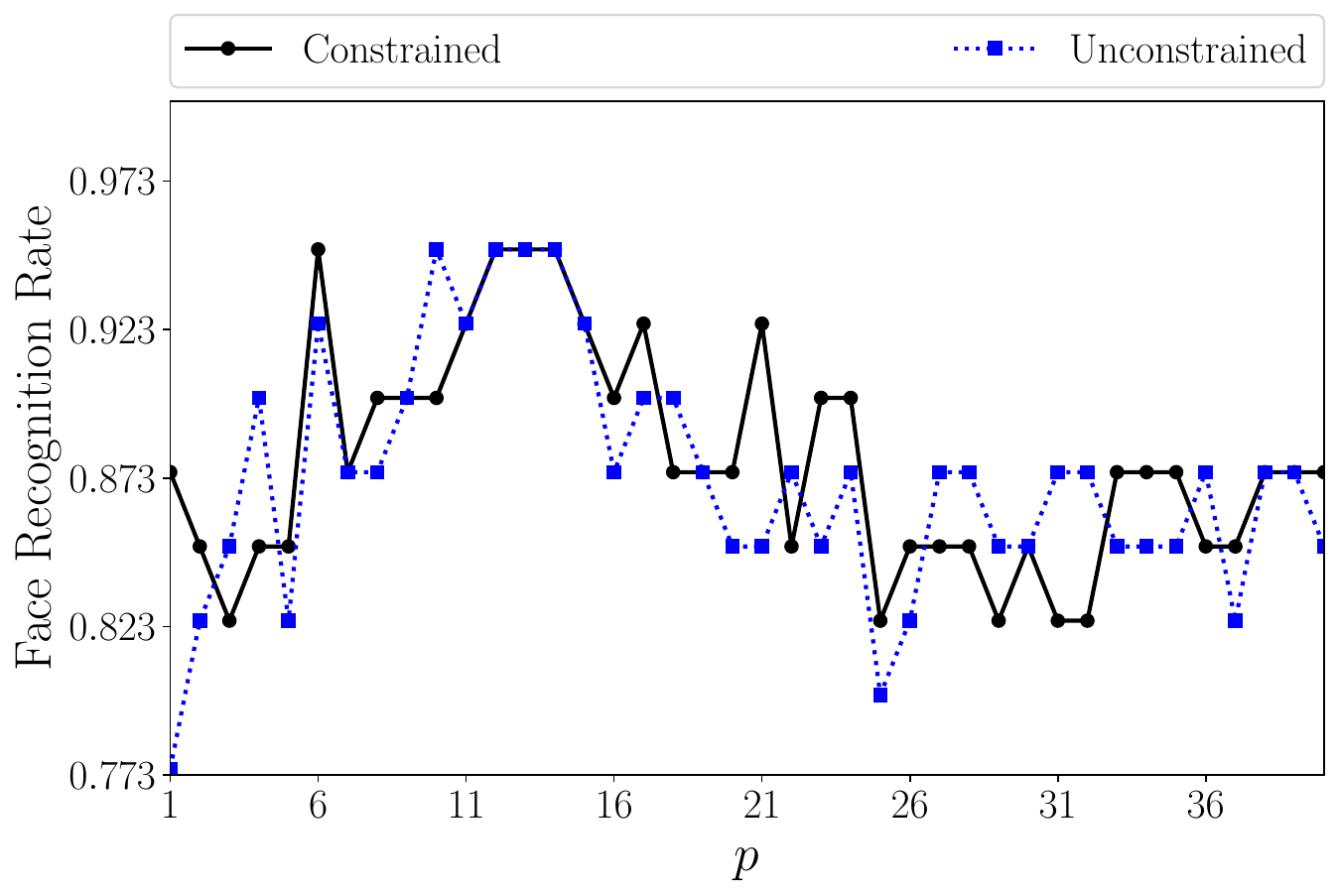}
	%\vspace{-1.5em}
	\caption{Example 3: Grassmannian diffusion maps -- Recognition rate as a function of $p$ for the recognition approach informed by the GDMaps.}
	\vspace{-0.5em}
	\label{fig:27}
\end{figure}

\section*{Concluding remarks}
\label{Cr}
In this paper, a novel dimensionality reduction technique, referred to as Grassmannian diffusion maps, was developed based on the concepts of Grassmann manifold and diffusion maps aiming at the characterization and classification of high-dimensional data in low-dimensional spaces. We demonstrated that an element of a dataset residing in a high-dimensional ambient space can be represented by a subspace constrained on the Grassmann manifold. The proposed GDMaps method is composed of two main steps, a pointwise linear dimensionality reduction and a multipoint nonlinear kernel-based dimensionality reduction. In the pointwise dimensionality, each element of a dataset is projected onto a Grassmann manifold where only the most relevant features defining the subspace where the data lie are retained. The subsequent multipoint dimensionality reduction applies the diffusion maps over the points on the Grassmann manifold to reveal the underlying structure of the points (subspaces) on the Grassmann manifold. Therefore, instead of using the information in the ambient space, which can be more susceptible to data corruption and external factors (e.g., noise, illumination level in images), the Grassmannian diffusion maps uses the information provided by the underlying lower-dimensional subspace of each data point. We further propose an algorithm to leverage the proposed GDMaps for classification of high-dimensional data. Along with a presentation of the proposed method, we provide the relevant details underpinning the method; in particular a careful investigation of kernels on the Grassmannian is presented to place the GDMaps on a firm foundation.

The performance of the Grassmannian diffusion maps was assessed in three examples. The first one contained a theoretical experiment to verify the ability of the GDMaps to identify a structured connection among points known to lie on a constrained submanifold of the Grassmann manifold, represented by the unit sphere in $\mathbb{R}^3$. It was shown that the Grassmannian diffusion maps can identify a well-defined parametric structure of the data when they are projected on the Grassmann manifold $\mathcal{G}(1,3)$. In the second example, a classification experiment with high-dimensional random field data was performed. The Grassmannian diffusion maps was applied to a dataset of matrices formed by oscillatory functions with random phase and frequency. A k-means was used to cluster the Grassmannian diffusion coordinates corresponding to matrices composed of oscillatory functions with same frequency. The third example presented a practical application of both the Grassmannian diffusion maps and the Grassmannian diffusion maps-based data classification using sparse representation. It was demonstrated using the AT\&T Database of Faces (AT\&T Laboratories Cambridge) that the Grassmannian diffusion maps is robust in identifying 40 face images subject to varying illumination conditions, change in face expressions, and occurrence of occlusions. Moreover, the developed technique presented high recognition rates ($95\%$ in the best-case) using low-dimensional data.

\section*{Acknowledgements} This work has been supported by the U.S. Department of Energy under grant number DE-SC0020428 and by internal funds provided by Johns Hopkins University.\\

\bibliographystyle{siamplain}
\bibliography{references}

\appendix

\section{Grassmannian kernel dimensionality}
\label{a:grassmannian_kernel_dimensionality}
Some properties of the projection and Binet-Cauchy kernels that may influence their selection for a specific application are investigated in this appendix. Specifically, the relationship between the dimensionality of the Grassmann manifold $\mathcal{G}(p,n)$, and the entries of the kernel matrix, is analyzed based on the distribution of the principal angles between random matrices. In this regard, let's first assume $p=1$ in Eqs. (\ref{eq:4.2} - \ref{eq:4.5}). One can easily show that $k_{pr}(\mathbf{\Psi}_i,\mathbf{\Psi}_j) = k_{bc}(\mathbf{\Psi}_i,\mathbf{\Psi}_j)$ for two different subspaces $\mathcal{X}_i=\mathrm{span}\left(\mathbf{\Psi}_i\right)$ and $\mathcal{X}_j=\mathrm{span}\left(\mathbf{\Psi}_j\right)$. However, more attention should be paid to the influence of $p$ and $n$, when $1< p < n$ and $n > 1$, on the affinity between distinct points on the same Grassmann manifold. Next, a more detailed analysis is presented for both the projection and Binet-Cauchy kernels. 

The lemmas presented in this section are developed considering two subspaces $\mathcal{X}_i=\mathrm{span}\left(\mathbf{\Psi}_i\right)$ and $\mathcal{X}_j=\mathrm{span}\left(\mathbf{\Psi}_j\right)$ chosen from the invariant uniform distribution on $\mathcal{G}(p,n)$. The entries of the kernel matrix are denoted by $k_{ij}=k(\mathbf{\Psi}_i,\mathbf{\Psi}_j)$, where for an ensemble of random subspaces the expected value is denoted by $\bar{k}_{ij}$.

\subsection{Projection kernel}
The following lemmas show that the expected values of the entries of $k_{pr}(\mathbf{\Psi}_i,\mathbf{\Psi}_j)$ have a well-defined functional relationship with both $p$ and $n$. 

\begin{lemma}\label{lem:4.3}
Given two random subspaces $\mathcal{X}_i = \mathrm{span}\left(\mathbf{\Psi}_i\right)$ and $\mathcal{X}_j = \mathrm{span}\left(\mathbf{\Psi}_j\right)$ on $\mathcal{G}(p,n)$, if $i=j$ the entries in the diagonal of $k_{pr}(\mathbf{\Psi}_i,\mathbf{\Psi}_j)$ are given by $k_{ii}=p$.\\
\end{lemma}
\begin{proof}
The proof of lemma \ref{lem:4.3} is trivial because $\mathrm{cos}^2(\theta_i) = 1$ for $i=1, \dots, p$ in Eq. (\ref{eq:4.5}).
\end{proof}
\noindent
On the other hand, if $i \neq j$, the following lemma holds.
\begin{lemma}\label{lem:4.4}
Given two random subspaces $\mathcal{X}_i = \mathrm{span}\left(\mathbf{\Psi}_i\right)$ and $\mathcal{X}_j = \mathrm{span}\left(\mathbf{\Psi}_j\right)$ on $\mathcal{G}(p,n)$, if $1 \leq p \leq n$ the expected value $\bar{k}_{ij}(p,n)$ of the off-diagonal entries of $k_{pr}(\mathbf{\Psi}_i,\mathbf{\Psi}_j)$ has the following functional form
\begin{equation}\label{eq:4.6}
    \bar{k}_{ij}(p,n) = \frac{p^2}{n}.
\end{equation}
\end{lemma}
\begin{proof}
Consider a fixed subspace $\mathcal{X}_i = \mathrm{span}(\mathbf{\Psi}_i)$ and a subspace $\mathcal{X}_j = \mathrm{span}(\mathbf{\Psi}_j)$ chosen from a uniform distribution on the Grassmann manifold given by
\begin{equation}\label{eq:4.7}
    \mathbf{\Psi}_i = \begin{bmatrix}
           \mathbf{I}_p \\
           \mathbf{0}_{n-p,p}
         \end{bmatrix},
\end{equation}
\noindent
and
\begin{equation}\label{eq:4.8}
    \hat{\mathbf{\Psi}}_j = \begin{bmatrix}
           \mathbf{A} \\
           \mathbf{B}
         \end{bmatrix},
\end{equation}
\noindent
where $\mathbf{I}_p$ is a $p \times p$ identity matrix, $\mathbf{0}_{n-p,p}$ is a $(n-p) \times p$ null matrix, $\mathbf{A} \in \mathbb{R}^{p \times p}$ and $\mathbf{B} \in \mathbb{R}^{(n - p) \times p}$ are i.i.d. Gaussian matrices, since the Gaussian distribution is invariant under orthogonal group transformation \cite{absil2006}. The orthonormalization of $\hat{\mathbf{\Psi}}_j$ is given by $\mathbf{\Psi}_j = \hat{\mathbf{\Psi}}_j\left(\mathbf{A}^T\mathbf{A} + \mathbf{B}^T\mathbf{B} \right)$. Thus, $\sigma_i^2 = \mathrm{cos}^2(\theta_i)$, with $i=1, \dots, p$, are equal to the eigenvalues values $\{\lambda_i\}_{i=1}^p$ of $\left(\mathbf{A}^T\mathbf{A} + \mathbf{B}^T\mathbf{B} \right)^{-1/2}\mathbf{A}^T\mathbf{A}\left(\mathbf{A}^T\mathbf{A} + \mathbf{B}^T\mathbf{B} \right)^{-1/2}$. Or equivalently given by the eigenvalues of $\mathbf{W} = (\mathbf{L}^{-1})^T\mathbf{A}^T\mathbf{A}\mathbf{L}^{-1}$, where $\mathbf{A}^T\mathbf{A} + \mathbf{B}^T\mathbf{B} = \mathbf{L}^T\mathbf{L}$ is the Cholesky decomposition of $\mathbf{A}^T\mathbf{A} + \mathbf{B}^T\mathbf{B}$. As $\mathbf{L}$ has a beta distribution $\mathrm{Beta}_p[p/2,(n-p)/2]$, the joint probability density function (PDF) of $\{\lambda_i\}_{i=1}^p$, when $n \geq 2p$, is given by \cite{muirhead1982,edelman2005,absil2006}

\begin{equation}\label{eq:4.9}
    f(\lambda_1, \dots, \lambda_p) = \frac{\pi^{p^2/2}\Gamma_p(n/2)}{\Gamma_p^2(p/2)\Gamma_p((n-p)/2)} \prod_{i<j}|\lambda_i - \lambda_j|\prod_{i=1}^p\lambda_i^{-1/2}(1 - \lambda_i)^{\frac{1}{2}(n - 2p - 1)},
\end{equation}
\noindent
where $\Gamma_m(\cdot)$ is the multivariate gamma function \cite{muirhead1982,absil2006}. Therefore, the PDF of the largest principal angle between the subspaces $\mathcal{X}_i$ and $\mathcal{X}_j$ randomly chosen from $\mathcal{G}(p,n)$ can be obtained (see \cite{absil2006} for a detailed presentation). At this point, our interest is focused on the case $p=1$, whose PDF of the cosine square of the unique principal angle is given by
\begin{equation}\label{eq:4.10}
    f(\lambda) = \frac{\Gamma\left(\frac{n}{2}\right)}{\Gamma\left(\frac{1}{2}\right)\Gamma\left(\frac{n-1}{2}\right)}\lambda^{-\frac{1}{2}}\left[1 - \lambda\right]^{\frac{n-3}{2}},
\end{equation}
\noindent
where $\Gamma(\cdot)$ is the gamma function, $\lambda = \mathrm{cos}^2(\theta_1)$. The mean and variance are given by 
\begin{equation}\label{eq:4.11}
    \mathrm{E}[\lambda] = \frac{1}{n},
\end{equation}
\noindent
and
\begin{equation}\label{eq:4.12}
    \mathrm{Var}[\lambda] = \frac{2(n-1)}{n^2 (n+2)}.
\end{equation}
\noindent
Assuming that $\mathbf{\Psi}^T_i\mathbf{\Psi}_j = \mathbf{USV^T}$, where $\mathbf{S} = \mathrm{diag}\left(\left[\sigma_1, \dots, \sigma_p\right]\right)$, and considering that $\mathbf{\Psi}_i$ and $\mathbf{\Psi}_j$ are orthonormal matrices, one can use the following identity
\begin{equation}\label{eq:4.13}
    \sum^p_{i=1}\sigma^2_i = \mathrm{Tr}\left(\mathbf{\Psi}_i\mathbf{\Psi}^T_i\mathbf{\Psi}_j\mathbf{\Psi}^T_j\right).
\end{equation}
\noindent
Considering that $\mathbf{Z}=\mathbf{\Psi}_i\mathbf{\Psi}^T_i\mathbf{\Psi}_j\mathbf{\Psi}^T_j$ one can write Eq. (\ref{eq:4.13}) for every realization $\xi$ of $\mathbf{\Psi}_j$, such that
\begin{equation}\label{eq:4.14}
    \sum^p_{i=1}\sigma^2_i(\xi) = \sum^p_{i=1}\mathbf{Z}_{ii}(\xi).
\end{equation}
\noindent
Taking the expectation of both sides of Eq. (\ref{eq:4.14}) one can obtain
\begin{equation}\label{eq:4.15}
    \sum^p_{i=1}\mathrm{E}[\sigma^2_i(\xi)] = \sum^p_{k=1}\mathrm{E}[\mathbf{Z}_{kk}(\xi)].
\end{equation}
\noindent
From the definition of $\mathbf{Z}$ one can obtain
\begin{equation}\label{eq:4.16}
    \mathbf{Z}_{kk}(\xi) = \sum^p_{l=1} \mathbf{\Psi}^2_{j,(k,l)},
\end{equation}
\noindent
where $\mathbf{\Psi}_{j,(k,l)}$ corresponds to the element $(k,l)$ of $\mathbf{\Psi}_{j}$. Therefore, as the columns of $\mathbf{\Psi}_j$ are orthonormal vectors in $\mathbb{R}^n$, the components of the vector $\mathbf{\Psi}_{j,(:,l)}$ are equal to the cosine of the direction angles $\alpha_k$ between $\mathbf{\Psi}_j$ and the component in the coordinate axis $k$ with $k = 1, \dots, n$. It is clear that the dot product of unit vectors uniformly distributed on the sphere $\mathbb{S}^n$ has a beta distribution; thus, assuming that $\mathbf{\Psi}_{j,(k,l)}$ are i.i.d. one can find that $\left(\mathbf{\Psi}_{j,(k,l)}+1\right)/2 \sim \mathrm{Beta}\left[(n-1)/2,(n-1)/2\right]$ \cite{gu2018}. As $\mathrm{E}\left[\mathbf{\Psi}_{j,(k,l)}\right] = 0$, the variance of $\mathbf{\Psi}_{j,(k,l)}$ is given by
\begin{equation}\label{eq:4.17}
    \mathrm{Var}\left[\mathbf{\Psi}_{j,(k,l)}\right] = \mathrm{E}\left[\mathbf{\Psi}^2_{j,(k,l)}\right] = \frac{1}{n}.
\end{equation}
\noindent
Moreover, if the columns of $\mathbf{\Psi}_j$ are orthonormal vectors in $\mathbb{R}^n$, the components of the vector $\mathbf{\Psi}_{j,(:,l)}$ are equal to the cosine of the direction angles $\alpha_k$ between $\mathbf{\Psi}_j$ and the component in the coordinate axis $k$ with $k = 1, \dots, n$. Therefore, $\mathbf{\Psi}^2_{j,(k,l)} = \mathrm{cos}^2(\alpha_k)$. Thus, one can observe that the results presented in Eqs. (\ref{eq:4.10} - \ref{eq:4.12}) are valid for $\lambda = \mathrm{cos}^2(\alpha_k)$.
Therefore, from Eq. (\ref{eq:4.16}) one can show that
\begin{equation}\label{eq:4.18}
    \mathrm{E}\left[\mathbf{Z}_{kk}(\xi)\right] = \frac{p}{n}.
\end{equation}
\noindent
Substituting Eq. (\ref{eq:4.18}) in Eq. (\ref{eq:4.15}) one can easily show that
\begin{equation}\label{eq:4.19}
    \bar{k}_{ij}(p,n) = \sum^p_{i=1}\mathrm{E}[\sigma^2_i(\xi)] = \frac{p^2}{n}.
\end{equation}
\end{proof}

\subsection{Binet-Cauchy kernel}
From Eq. (\ref{eq:4.2}) the Binet-Cauchy kernel corresponds to the product of the square of the singular values of $\mathbf{\Psi}^T_i\mathbf{\Psi}_j$. In this regard, the values of the off-diagonal elements of $k_{bc}(\mathbf{\Psi}_i,\mathbf{\Psi}_j)$ are governed by those singular values larger than zero and lower than one. Therefore, one can start this analysis assuming that the following lemma holds.

\begin{lemma}\label{lem:4.5}
Given two random subspaces $\mathcal{X}_i$ and $\mathcal{X}_j$ on $\mathcal{G}(p,n)$, if $p \geq n/2$ the multiplicity of the principal angle $\theta = 0$ between them is equal to $2p - n$.\\
\end{lemma}
\begin{proof}
Assuming that $\mathcal{X}_i=\mathrm{span}\left(\mathbf{\Psi}_i\right)$ and $\mathcal{X}_j=\mathrm{span}\left(\mathbf{\Psi}_j\right)$, and considering that the multiplicity of the principal angle $\theta = 0$ between $\mathcal{X}_i$ and $\mathcal{X}_j$ is equal to $\mathrm{dim}(\mathcal{X}_i \cap \mathcal{X}_j)$, one can use the following expression

\begin{equation}\label{eq:4.20}
    \mathrm{dim}(\mathcal{X}_i \cap \mathcal{X}_j) = \mathrm{dim}(\mathcal{X}_i) + \mathrm{dim}(\mathcal{X}_j) - \mathrm{dim}(\mathcal{X}_i + \mathcal{X}_j),
\end{equation}
\noindent
or alternatively,
\begin{equation}\label{eq:4.21}
    \mathrm{dim}(\mathcal{X}_i \cap \mathcal{X}_j) = \mathrm{rank}(\mathbf{\Psi}_i) + \mathrm{rank}(\mathbf{\Psi}_j) - \mathrm{rank}(\left[\mathbf{\Psi}_i, \mathbf{\Psi}_j\right]).
\end{equation}
\noindent
As $\mathcal{X}_i, \mathcal{X}_j \in \mathcal{G}(p,n)$, $\mathrm{rank}(\mathbf{\Psi}_i) = \mathrm{rank}(\mathbf{\Psi}_j) = p$. Moreover, $\mathrm{rank}(\left[\mathbf{\Psi}_i, \mathbf{\Psi}_j\right]) = n$, since $p \geq n/2$. Therefore, $\mathrm{dim}(\mathcal{X}_i \cap \mathcal{X}_j) = 2p - n$.\\
\end{proof}
\noindent
The result presented in lemma \ref{lem:4.5} is useful to show that $p = n/2$ corresponds to the largest value of $p$ where all the singular values of $\mathbf{\Psi}^T_i\mathbf{\Psi}_j$ are strictly less than 1. Therefore, the number of positive singular values smaller than one is equal to $p$ if $1 \leq p < n/2$; and $n-p$ if $n/2 \leq p < n$. Thus, one can conclude that when $p=n/2$, the expected value of the off-diagonal entries $\bar{k}_{ij}(p,n)$ of $k_{bc}(\mathbf{\Psi}_i,\mathbf{\Psi}_j)$ is minimal. Next, a lemma for the diagonal elements of the Binet-Cauchy kernel matrix $k_{bc}$ is presented.

\begin{lemma}\label{lem:4.6}
Given two random subspaces $\mathcal{X}_i = \mathrm{span}\left(\mathbf{\Psi}_i\right)$ and $\mathcal{X}_j = \mathrm{span}\left(\mathbf{\Psi}_j\right)$ on $\mathcal{G}(p,n)$, if $i=j$ the entries in the diagonal of $k_{bc}(\mathbf{\Psi}_i,\mathbf{\Psi}_j)$ are given by $k_{ii}=1$.\\
\end{lemma}
\begin{proof}
The proof of lemma \ref{lem:4.6} is trivial because $\mathrm{cos}^2(\theta_i) = 1$ for $i=1, \dots, p$ in Eq. (\ref{eq:4.3}).
\end{proof}

One can start the analysis of the off-diagonal entries of the Binet-Cauchy kernel matrix from the trivial cases. Considering two random subspaces $\mathcal{X}_i$ and $\mathcal{X}_j$ one can easily show that for $p=1$, the expected value of the off-diagonal entries of the Binet-Cauchy kernel is given by $\bar{k}_{ij}(p,n) = 1/n$, as in the projection kernel. On the other hand, for $p=n$, which is an extreme case used for theoretical purposes only, one can observe that $k_{ij}(p)=1$ because all the principal angles are equal to zero. Therefore, in the extreme cases both kernels have similar behavior. More generally, one can show that the expected values of the off-diagonal entries of the Binet-Cauchy kernel can be obtained using the joint probability density function in Eq. (\ref{eq:4.9}); however, this calculation is cumbersome. Alternatively, one can define an upper bound for the expected value of the off-diagonal entries of the Binet-Cauchy kernel matrix as presented in the following lemma.

\begin{lemma}\label{lem:4.7}
Given two random subspaces $\mathcal{X}_i = \mathrm{span}\left(\mathbf{\Psi}_i\right)$ and $\mathcal{X}_j = \mathrm{span}\left(\mathbf{\Psi}_j\right)$ on $\mathcal{G}(p,n)$, the expected value $\bar{k}_{ij}(p,n)$ of the off-diagonal entries of $k_{bc}(\mathbf{\Psi}_i,\mathbf{\Psi}_j)$ has the following upper bound

\begin{equation}\label{eq:4.22}
  \bar{k}_{ij}(p,n) \leq \left \{
  \begin{aligned}
    &\left(\frac{p}{n}\right)^{p}, && \text{if}\ 1 \leq p < \frac{n}{2} \\
    &\left(\frac{n-p}{n}\right)^{n-p}, && \text{if}\ \frac{n}{2} \leq p < n
  \end{aligned} \right. .
\end{equation}
\end{lemma}
\begin{proof}
Considering the matrices $\mathbf{\Psi}_i$ and $\mathbf{\Psi}_j$, presented in the proof of lemma \ref{lem:4.4}, as bases of both subspaces $\mathcal{X}_i$ and $\mathcal{X}_j$, respectively; and using the inequality of arithmetic and geometric means (AM-GM inequality) \cite{courant1996} one can observe that
\begin{equation}\label{eq:4.23}
    \left(\prod^p_{i=1}\mathrm{E}[\sigma^2_i(\xi)] \right)^{\frac{1}{p}} \leq \frac{1}{p}\sum^p_{i=1}\mathrm{E}[\sigma^2_i(\xi)].
\end{equation}
\noindent
From lemma \ref{lem:4.4} one can say that $\mathrm{E}[\sigma^2_i(\xi)]$ with $i=1, \dots, p$ is equal to $p/n$. Thus, one can write
\begin{equation}\label{eq:4.24}
    \left(\prod^p_{i=1}\mathrm{E}[\sigma^2_i(\xi)] \right)^{\frac{1}{p}} \leq \frac{p}{n}.
\end{equation}
\noindent
Therefore, the expected value of the off-diagonal entries of the Binet-Cauchy kernel has an upper bound given by
\begin{equation}\label{eq:4.25}
    \bar{k}_{ij}(p,n) \leq \left(\frac{p}{n}\right)^{p},
\end{equation}
\noindent
if $p<n/2$. On the other hand, if $n/2 \leq p < n$, the upper bound is given by
\begin{equation}\label{eq:4.26}
    \bar{k}_{ij}(p) \leq \left(\frac{n-p}{n}\right)^{n-p}.
\end{equation}
\noindent 
In fact, Eq. (\ref{eq:4.26}) holds because for $n/2 \leq p < n$ only $n-p$ singular values are in the interval $(0,1)$ and contribute to the product in the left-hand side of Eq. (\ref{eq:4.23}).
\end{proof}

%\bibliography{sample.bib}
\end{document}